\documentclass[sigconf]{acmart}
\AtBeginDocument{%
  }

\setcopyright{acmlicensed}
\copyrightyear{2018}
\acmYear{2018}
\acmDOI{XXXXXXX.XXXXXXX}
\acmConference[Conference acronym 'XX]{Make sure to enter the correct
  conference title from your rights confirmation email}{June 03--05,
  2018}{Woodstock, NY}
\acmISBN{978-1-4503-XXXX-X/2018/06}

\newtheorem{theorem}{Theorem}[section]
\newtheorem{corollary}{Corollary}[theorem]
\newtheorem{lemma}[theorem]{Lemma}
\usepackage{algpseudocode}
\usepackage{stfloats}
\usepackage{amsmath,amsthm}
\usepackage{algorithm}
\usepackage{algpseudocode}

\begin{document}

\title{RSPO: Risk-Seeking Policy Optimization for Pass@k and Max@k Metrics in Large Language Models}
\author{%
  Kaichen Zhang$^{1,2}$ \quad Shenghao Gao$^{1,2}$ \quad Yuzhong Hong$^{2}$ \quad Haipeng Sun$^{2}$ \quad Junwei Bao$^{2}$ \\ Hongfei Jiang$^{2}$ \quad
 Yang Song$^{2}$ \quad Dingqian Hong$^{2}$ \quad Hui Xiong$^{1}$ \\
  $^{1}$Hong Kong University of Science and Technology (Guangzhou) \\
  $^{2}$Zuoyebang Education Technology \\
  % \texttt{kzhangbi@connect.ust.hk} \quad \texttt{baojunwei001@gmail.com} \quad \texttt{xionghui@ust.hk}\\
  % \texttt{$\{$hongyuzhong,sunhaipeng01,jianghongfei,songyang,hongdingqian$\}$@zuoyebang.com}
}
\renewcommand{\shortauthors}{Trovato et al.}

\begin{abstract}

Current large language model post-training optimizes a \textit{risk-neutral} objective that maximizes expected reward, yet evaluation relies heavily on \textit{risk-seeking} metrics like Pass@k (at least one success in $k$ trials) and Max@k (maximum reward across $k$ responses). This mismatch in risk preferences can inevitably lead to suboptimal performance. To bridge this gap, we propose \textbf{Risk-Seeking Policy Optimization (RSPO)}, a novel method that directly targets Pass@k and Max@k during training. A key challenge in optimizing these metrics is the "hitchhiking" problem: low-reward responses are inadvertently reinforced if they co-occur with a high-reward response within a sample of $k$ generations, resulting in inefficient optimization. RSPO addresses this problem by leveraging the closed-form probability that a given response is the maximum among $k$ samplings. Despite the complexity of nested gradients over multiple responses, RSPO produces efficient, unbiased gradient estimators for both metrics. We validate our approach with both rigorous theoretical analysis and comprehensive experimental results.

\end{abstract}



\keywords{Large Language Model, Post Training, Pass@k, Max@k}
  \begin{teaserfigure}
  % \vspace{+12mm}
  \end{teaserfigure}


\maketitle
\newpage
\begin{figure}[t]
\centering
\includegraphics[width=0.47\textwidth]{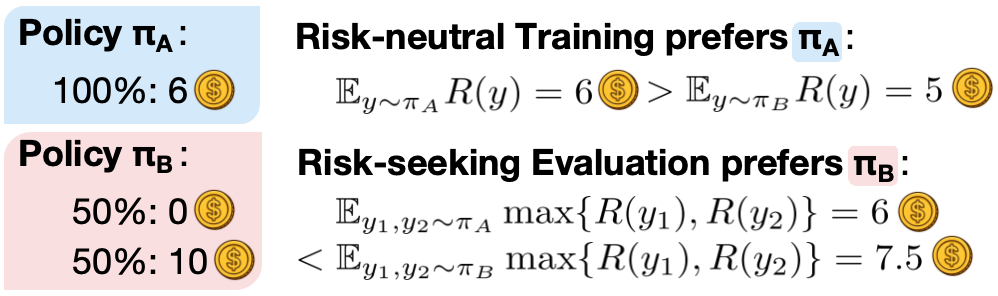}
% \vspace*{-5mm}
\caption{The risk preference mismatch between training and evaluation leads to different policy preferences.}
\label{fig:intro}
\vspace{-3mm}
\end{figure}

\section{Introduction}
Large language models (LLMs) \cite{zhao2025surveylargelanguagemodels,minaee2025largelanguagemodelssurvey} have demonstrated remarkable capabilities in natural language understanding and generation, driven by architectures such as the Transformer \cite{vaswani2017attention} and scaled-up pre-training \cite{zhou2023comprehensivesurveypretrainedfoundation} on massive text corpora. Post-training \cite{tie2025surveyposttraininglargelanguage} techniques—such as supervised fine-tuning (SFT) \cite{ouyang2022training} and most notably reinforcement learning (RL) \cite{sutton1998reinforcement} have emerged as a powerful paradigm to tailor models toward specific tasks and human preferences. By defining a reward function that captures desired behaviors (e.g., helpfulness, correctness), reinforcement learning can steer LLM outputs toward higher-quality responses and align them more closely with end-user needs \cite{bai2022training}.

Despite these successes, a fundamental mismatch persists between the objectives used during RL training and the metrics by which LLMs are evaluated. Standard reinforcement learning maximizes the expected reward across all generations—an inherently \textit{risk-neutral} objective—whereas practical evaluation often hinges on \textit{risk-seeking} 
% \footnote{Risk-seeking refers to a preference for high-risk, high-return options.} 
metrics such as Pass@k (the probability that at least one of $k$ samples is correct) or Max@k (the highest reward obtained among $k$ sampled responses). For example, \cite{stiennon2020learning} employs an inference-time strategy that selects the optimal output from a set of generated candidates, while \cite{yue2025does} uses Pass@k as a metric to assess the inherent reasoning capabilities of LLMs.

This divergence in risk preference can lead to suboptimal behavior: policies tuned to maximize average reward may under-explore high-variance, high-reward responses, resulting in degraded Max@k performance. Figure~\ref{fig:intro} demonstrates the risk preference mismatch, in which training prefers $\pi_A$ but evaluation prefers $\pi_B$.

A critical obstacle to bridging this gap is the so-called "hitchhiking" phenomenon. When optimizing Pass@k or Max@k, low-reward responses that co-occur with a single high-reward sample within a batch of $k$ can receive undue positive reinforcement, illustrated in Figure~\ref{fig:intro_2}. Such hitchhiking skews the gradient signal away from genuinely promising responses, leading to inefficient learning and slower convergence toward optimal policies.

\newpage

\begin{figure}[t]
\centering
\includegraphics[width=0.35\textwidth]{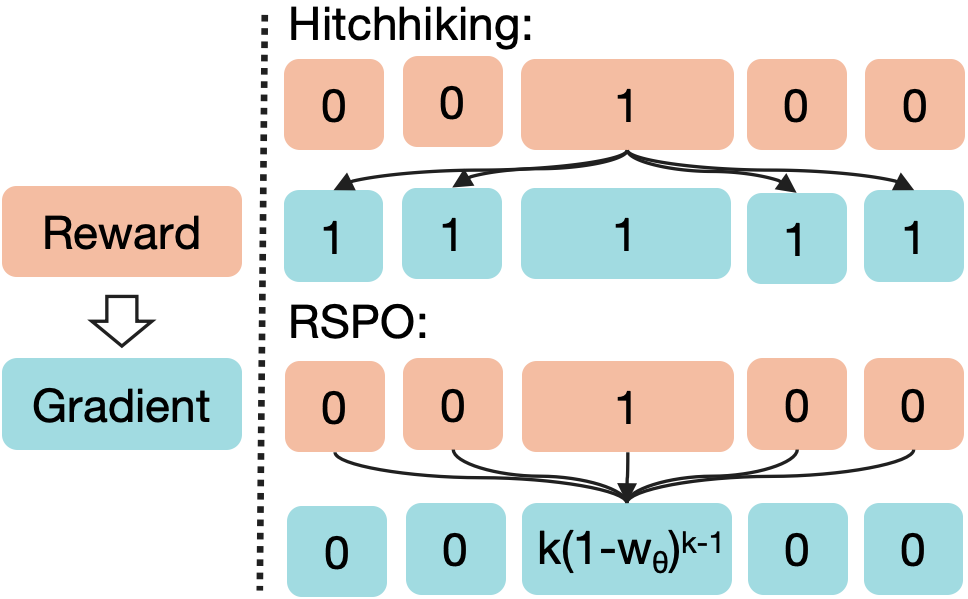}
% \vspace*{-5mm}
\caption{An illustration comparing "hitchhiking" baseline and RSPO for computing gradient weights from rewards.}
\label{fig:intro_2}
\vspace{-5mm}
\end{figure}

To address these challenges, we propose Risk-Seeking Policy Optimization (RSPO), a novel RL algorithm that directly targets Pass@k and Max@k metrics. 

RSPO derives a closed-form expression for the probability that any given response has the maximum reward in a set of $k$ samples. This expression decouples the joint distribution of multiple responses in the Pass@k and Max@k objectives into a distribution over a single response, circumventing the "hitchhiking" issue.

Utilizing this expression, we derive a simplified objective for the Pass@k objective, wherein rewards are binary-valued:
$$
\mathop{\mathbb{E}}_{x \sim \mathcal{D}}\left[\sum_{y}\pi_\theta(y|x)\frac{1-(1-w_\theta)^k}{w_\theta}R(x,y)\right]
$$
In this objective, $w_\theta$ represents the probability that policy $\pi_\theta$ generates a reward of one. Although this formulation initially appears to involve nested gradients across multiple responses (due to the gradient term $\nabla_\theta(\pi_\theta(y|x)\frac{1-(1-w_\theta)^k}{w_\theta})$), we demonstrate that the gradient of the Pass@k objective can be elegantly simplified to:
$$
\mathop{\mathbb{E}}_{x \sim \mathcal{D},y \sim \pi_\theta(y|x)}\left[k(1-w_\theta)^{k-1}R(x,y) \nabla_\theta\log\pi_\theta(y|x)\right]
$$

The resulting gradient weight, $k(1-w_\theta)^{k-1}$, naturally decreases as $w_\theta$ increases for $k \geq 2$. This property significantly saves \textit{opportunity cost} during optimization. Specifically, if the policy consistently generates correct responses within $k$ attempts, further reinforcing correct answers becomes unnecessary. Instead, probabilities are more effectively allocated towards exploring alternative tokens, responses, or patterns that promise a higher global payoff.

For Max@k, we also derive a simplified objective without the "hitchhiking" issue. This simplified objective conveys the meaningful concept of \textit{marginal contribution}. Specifically, the gradient weight for a response \( y \) reflects the marginal gain in the maximum reward when \( y \) is added to sets of \( k-1 \) responses.

Based on prior derivations, we can construct unbiased estimators for the Pass@k and Max@k gradient weights. For example, $(1 - w_\theta)^{k-1}$ can be estimated using $\binom{n - c}{k - 1}/\binom{n - 1}{k - 1}$, where $n$ and $c$  denote the number of total and correct samples, respectively.

We conduct a comprehensive evaluation of RSPO on math reasoning. Our results demonstrate that RSPO consistently outperforms the baseline algorithm—which is hindered by the "hitchhiking" problem—across datasets and metrics. Notably, RSPO achieves optimal performance when the training hyperparameter $k$ is aligned with the evaluation metrics Pass@k and Max@k. Finally, we assess the scalability and robustness of RSPO, confirming its effectiveness under a wide range of conditions.

In summary, this paper makes three key contributions:
\begin{itemize}
    \item We identify the misalignment problem between standard risk‑neutral RL objective and risk‑seeking evaluation metrics, which can degrade LLM post‑training performance.

    \item We address the key "hitchhiking" challenge inherent to optimizing Pass@k and Max@k by proposing efficient, unbiased estimators for both metrics and providing rigorous theoretical analyses of their underlying mechanisms.
    
    \item We propose the RSPO algorithm, which, through extensive empirical evaluation, effectively maximizes Pass@k and Max@k metrics.
\end{itemize}

\section{Related Literature}
Post-training \cite{tie2025surveyposttraininglargelanguage} of large language models \cite{zhao2025surveylargelanguagemodels,minaee2025largelanguagemodelssurvey} refers to the refinement stage that follows initial pre-training \cite{zhou2023comprehensivesurveypretrainedfoundation}, aiming to adapt a generalized model to specific tasks or to align it with human-defined objectives \cite{bai2022training}. Supervised fine-tuning (SFT) \cite{ouyang2022training} is the most direct and widely adopted post-training method. It involves updating model parameters using labeled examples where the desired outputs for specific inputs are explicitly provided. Although straightforward and effective, SFT demands extensive human annotation, significantly limiting its scalability.

To address the limitations of SFT, reinforcement learning (RL) \cite{sutton1998reinforcement} methods have been increasingly utilized in post-training. RL-based post-training seeks to maximize the expected reward, formally represented as $\mathop{\mathbb{E}}_{x\sim\mathcal{D},y\sim\pi_\theta(y|x)}R(x,y)$. Among RL methods, policy gradient methods have emerged as the dominant framework. Policy gradient \cite{sutton1999policy} algorithms compute the gradient of the RL objective and optimize by sampling according to this gradient, resulting in the loss function $L=-R(x,y)\nabla_\theta\log\pi_\theta(y|x)$. Advanced algorithms such as Trust Region Policy Optimization (TRPO) \cite{pmlr-v37-schulman15} and Proximal Policy Optimization (PPO) \cite{schulman2017proximal} build upon this foundation by imposing constraints on policy updates, thereby ensuring stability and robustness during training. Most recently, Group Relative Policy Optimization (GRPO) \cite{shao2024deepseekmath, guo2025deepseek} has been proposed, wherein multiple responses are sampled, and standard scores of their rewards are employed as gradient weights, thereby eliminating the need for training an additional model for \textit{advantage} estimation.

Despite the accuracy metric (i.e., Pass@1), the evaluation of LLM using verifiable rewards \cite{wang2025survey} often employs the Pass@k metric. Pass@k, especially when applied to reasoning tasks, is frequently interpreted as reflecting the inherent capacity of an LLM \cite{yue2025does}. 
Notably, the unbiased estimation method for Pass@k introduced by Chen et al. \cite{chen2021evaluating} is widely recognized in the LLM literature, although this approach closely aligns with classical U-statistics \cite{hoeffding1992class}.

Beyond evaluation metrics, selecting the optimal output among n generated candidates—referred to as the best-of-n method \cite{stiennon2020learning}— serves as an inference-time strategy. Despite its effectiveness, best-of-n is computationally expensive. Consequently, prior research \cite{gui2024bonbon, sessa2024bond} has primarily focused on distilling a new policy that approximates the best-of-n distribution derived from a fixed, given policy. In contrast, this paper diverges from such approaches by actively searching through an extensive policy space to directly optimize Pass@k or Max@k, rather than approximating a static policy.

\newpage
\section{Preliminary} \label{sec:preliminary}

\textbf{Notations}. Large language models process a prompt $x$ as input and generate a response $y$ as output. The policy $\pi_{\theta}$, parameterized by $\theta$, defines the model's behavior: $\pi_{\theta}(y_t | x, y_{<t})$ represents the probability of generating the next token $y_t$ given the prompt $x$ and the previously generated tokens $y_{<t}$. Furthermore, $\pi_{\theta}(y | x)$ denotes the overall probability of generating the complete response $y$ given the prompt $x$. To evaluate response quality, a reward model $R(x, y)$ assigns a score to the pair $(x, y)$. This reward can be binary (e.g., indicating correctness) or continuous (e.g., reflecting learned human preferences).

For reading clarity, we introduce a \textit{total order} on the set of responses $y$ by by defining $y_1\leq y_2$ if and only if $R(x,y_1)\leq R(x,y_2)$. Besides, the strict ordering $y_1<y_2$ implies $R(x,y_1)<R(x,y_2)$.

\noindent
\textbf{Problem Formulation.} We formally define the Max@k objective for large language model post-training: given an initial policy $\pi_{\theta_{0}}$, a dataset of prompts $x \sim D$, a reward model $R$, the objective is to train a new policy $\pi_{\theta}$ that maximizes 
\begin{equation}
\mathop{\mathbb{E}}\limits_{x\sim\mathcal{D},y_1,y_2,...,y_k\sim\pi_\theta(y|x)}[\max\{R(x,y_1),R(x,y_2),...,R(x,y_k)\}] 
\end{equation}

Moreover, when rewards are binary ($R(x,y)\in\{0,1\}$), we fine the Pass@k objective as 
\vspace{-3mm}
\begin{equation}
\mathop{\mathbb{E}}\limits_{x\sim\mathcal{D},y_1,y_2,...,y_k\sim\pi_\theta(y|x)}[1-\prod_{i=1}^k(1-R(x,y_i)] 
\end{equation}

\noindent
\textbf{Baseline.} A straightforward approach to optimizing Max@k is to treat the group $y_1,y_2,...,y_k$ as a whole and apply policy gradient:
\begin{equation} \nonumber
\begin{aligned}
    &\nabla_\theta\mathop{\mathbb{E}}\limits_{x\sim\mathcal{D},y_1,y_2,...,y_k\sim\pi_\theta(y|x)}\max_{1\le i\le k} R(x,y_i) \\
    &=\nabla_\theta\sum_x P(x)\sum_{y_{1:k}}\pi_\theta(y_{1:k}|x) \max_{1\le i\le k} R(x,y_i) \\
    &=\sum_x P(x)\sum_{y_{1:k}}\pi_\theta(y_{1:k}|x)\max_{1\le i\le k} R(x,y_i) \nabla_\theta \log\pi_\theta(y_{1:k}|x) \\
    &=\sum_x P(x)\sum_{y_{1:k}}\pi_\theta(y_{1:k}|x) \max_{1\le i\le k} R(x,y_i)  \sum_{i=1}^k \nabla_\theta\log\pi_\theta(y_i|x) \\
    &\approx \frac{1}{|\mathcal{D}_b|}\sum_{x\sim \mathcal{D}_b}\frac{1}{m}\sum_{j=1}^m \max_{1\le i\le k} R(x,y_i^j)  \sum_{i=1}^k \nabla_\theta\log\pi_\theta(y_i^j|x)
\end{aligned}
\end{equation}

The last line employs a Monte Carlo simulation to sample $|\mathcal{D}_b|$ prompts and generate $m$ groups of responses for each prompt. 

Within each group, all $k$ responses share the same gradient weight, given by $\max_{1 \le i \le k} R(x, y_i^j)$—the maximum reward among the $k$ responses in that group. This approach introduces what we refer to as the "hitchhiking" problem: low-reward responses may be inadvertently reinforced simply by co-occurring with a high-reward response within the same group. As a result, this can lead to inefficient optimization and suboptimal learning dynamics.

\begin{figure}[t]
\centering
\includegraphics[width=0.47\textwidth]{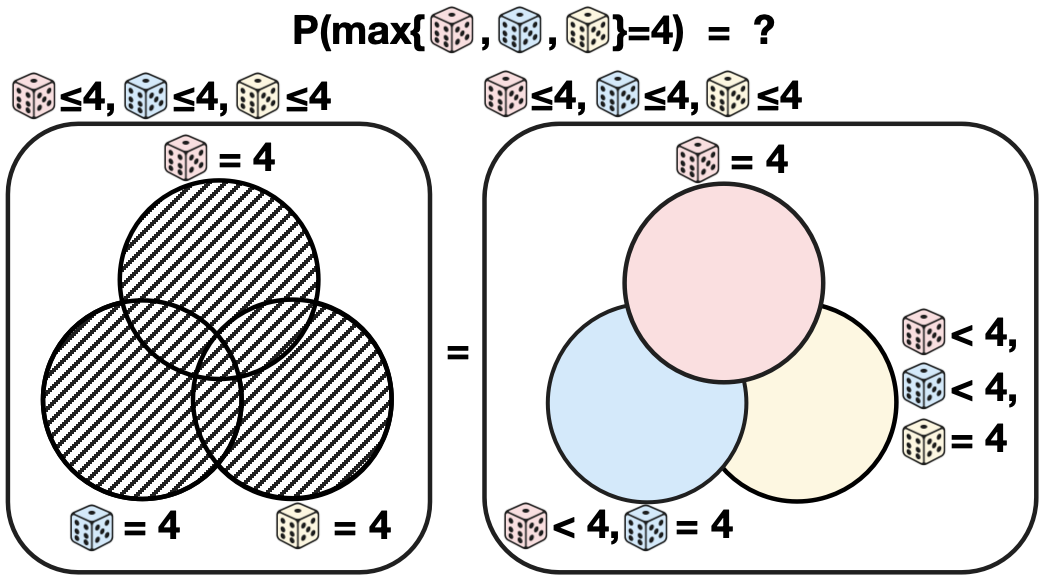}
% \vspace*{-5mm}
\caption{An illustration of how to calculate the probability of achieving the maximum in a group: $P(\max\{D_1,D_2,D_3\}=4)=P(D_1=4,D_2\leq4,D_3\leq4)+P(D_1<4,D_2=4,D_3\leq4)+P(D_1<4,D_2<4,D_3=4)=\frac{1}{6}\frac{2}{3}\frac{2}{3}+\frac{1}{2}\frac{1}{6}\frac{2}{3}+\frac{1}{2}\frac{1}{2}\frac{1}{6}=\frac{37}{216}.$}
\label{fig:disjoint}
% \vspace{-10mm}
\end{figure}

\section{Risk Seeking Policy Optimization}
\subsection{Insight}
To address the "hitchhiking" problem, it is essential to decouple individual responses from the sets in which they appear. RSPO addresses this challenge by computing the probability that a given response yields the maximum reward within a set of size $k$:

\begin{theorem} \label{theorem:bestk}
Let $y_1, y_2, \dots, y_k$ be independent samples drawn from the distribution $\pi_\theta(\cdot \mid x)$.  Then the probability that a specific response $y$ is the one with the highest reward among $\{y_i\}_{i=1}^k$ is
\begin{equation}
    \begin{aligned}
        &P(y \in\{y_i\}_{i=1}^k,R(x,y)=\max_{1\le i\le k}R(x,y_i))\\
        =&\sum_{i=1}^k P_{<,\theta}(y)^{i-1}\cdot \pi_\theta(y|x) \cdot P_{\leq,\theta}(y)^{k-i}
    \end{aligned}
\end{equation}
where $P_{<,\theta}(y)=\sum_{y'<y}\pi_\theta(y'|x)$ and $P_{\leq,\theta}(y)=\sum_{y'\leq y}\pi_\theta(y'|x)$.
% where $P_{<,\theta}(y)=\sum_{R(x,y')<R(x,y)}\pi_\theta(y'|x)$

 % and $P_{\leq,\theta}(y)=\sum_{R(x,y')\leq R(x,y)}\pi_\theta(y'|x)$.
\end{theorem}

% \begin{lemma}
% Let $y_1,y_2,...,y_k$ been independent random variables following $\pi_\theta(\cdot|x)$. Then the probability of response $y$ being the response with the maximum reward in the set $y_1,y_2,...,y_k$ is:
% \begin{equation}
%     \begin{aligned}
%         &P(y \in\{y_i\}_{i=1}^k,R(x,y)=\max\{R(x,y_i)\}_{i=1}^k)\\
%         =&\sum_{i=1}^k P_{<,\theta}(y)^{i-1}\cdot \pi_\theta(y|x) \cdot P_{\leq,\theta}(y)^{k-i}
%     \end{aligned}
% \end{equation}
% where $P_{<,\theta}(y)=\sum_{R(x,y')<R(x,y)}\pi_\theta(y'|x)$

%  and $P_{\leq,\theta}(y)=\sum_{R(x,y')\leq R(x,y)}\pi_\theta(y'|x)$.
% \end{lemma}
\begin{proof}
Observe that the event $y \in\{y_i\}_{i=1}^k$ indicates $R(x,y)\leq\max_{1\le i\le k}R(x,y_i)$. This implies:
\begin{equation} \nonumber
    \begin{aligned}
        &P(y \in\{y_i\}_{i=1}^k,R(x,y)=\max_{1\le i\le k}R(x,y_i) )\\
        =&P(y \in\{y_i\}_{i=1}^k,R(x,y)\geq\max_{1\le i\le k}R(x,y_i))\\
        =&P(y \in\{y_i\}_{i=1}^k,R(x,y)\geq R(x,y_i) \ \forall\,i)
    \end{aligned}
\end{equation}
Moreover, the event $y \in\{y_i\}_{i=1}^k$ can be decomposed as the disjoint union: 
$P(y \in\{y_i\}_{i=1}^k)=P(y_1=y)+P(y_1\neq y,y_2=y)+...+P(y_1\neq y,y_2\neq y,...,y_{k-1}\neq y,y_k=y)$. As a result,
\begin{equation} \nonumber
    \begin{aligned}
        &P(y_1\neq y,y_2\neq y,...,y_{i-1}\neq y,y_i=y,R(x,y)\geq R(x,y_i)\ \forall\,i) \\
        =&\prod_{j=1}^{i-1}P(R(x,y_j)<R(x,y))\cdot \pi_\theta(y|x)\cdot \prod_{j=1}^{i-1}P(R(x,y_j)\leq R(x,y))\\
        =&P_{<,\theta}(y)^{i-1}\pi_\theta(y|x)P_{\leq,\theta}(y)^{k-i}
    \end{aligned}
\end{equation}
Summing over \(i=1,\dots,k\) therefore yields Theorem \ref{theorem:bestk}:
    \begin{align*}
        &P(y \in\{y_i\}_{i=1}^k,R(x,y)=\max_{1\le i\le k}R(x,y_i))\\
        =&\sum_{i=1}^k P_{<,\theta}(y)^{i-1}\cdot \pi_\theta(y|x) \cdot P_{\leq,\theta}(y)^{k-i}  \qedhere
    \end{align*}
\end{proof}

Figure~\ref{fig:disjoint} provides a visual illustration of the proof. It demonstrates how the probability of being the maximum in a group can be computed by decomposing it into a disjoint union of several simpler, computable probabilities.

\begin{figure}[t]
\centering
\includegraphics[width=0.47\textwidth]{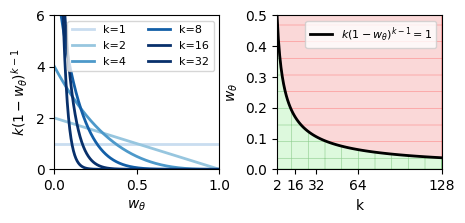}
% \vspace*{-5mm}
\caption{
(Left) Weight functions $k(1-w_\theta)^{k-1}$ as $w_\theta$ varies; each curve corresponds to a different $k$; (Right) Regions shaded in green indicate parameter pairs $(k,w_\theta)$ for which $k(1-w_\theta)^{k-1}>1$, while red shading marks $k(1-w_\theta)^{k-1}<1$.
}
\label{fig:passk}
% \vspace{-10mm}
\end{figure}

By applying Theorem~\ref{theorem:bestk}, we decouple the joint distribution of multiple responses in the Max@k metric into a distribution over a single response:

\begin{theorem} The original Max@k objective
    \begin{equation}\nonumber
\mathop{\mathbb{E}}\limits_{x\sim\mathcal{D},y_1,y_2,...,y_k\sim\pi_\theta(y|x)}[\max\{R(x,y_1),R(x,y_2),...,R(x,y_k)\}] 
    \end{equation}
    is equivalent to
    \begin{equation}
\mathop{\mathbb{E}}\limits_{x\sim\mathcal{D},y\sim\pi_\theta(y|x)}R(x,y)\sum_{i=1}^k P_{<,\theta}(y)^{i-1} P_{\leq,\theta}(y)^{k-i}
    \end{equation}
\end{theorem}

We further define the probability $w_\theta=\sum_{R(x,y)=1}\pi_\theta(y|x)$ to simplify the expression for the Pass@k objective:
\[
P_{<,\theta}(y)= \begin{cases}
    1-w_\theta & R(x,y)=1 \\
    0& R(x,y)=0
\end{cases}
\]
and
\[
P_{\leq,\theta}(y)= \begin{cases}
    1 & R(x,y)=1 \\
    1-w_\theta& R(x,y)=0
\end{cases}
\]
As a result, Theorem~\ref{theorem:bestk} also yields
\begin{corollary}
The original Pass@k objective
    \begin{equation}\nonumber
\mathop{\mathbb{E}}\limits_{x\sim\mathcal{D},y_1,y_2,...,y_k\sim\pi_\theta(y|x)}[1-\prod_{i=1}^k(1-R(x,y_i)] 
    \end{equation}
    is equivalent to
    \begin{equation}
\mathop{\mathbb{E}}\limits_{x\sim\mathcal{D},y\sim\pi_\theta(y|x)}R(x,y)\sum_{i=1}^k (1-w_\theta)^{i-1}
    \end{equation}
\end{corollary}

\subsection{RSPO for Pass@k} \label{sec:passk}
We first study how to optimize the Pass@k objective by directly applying the policy gradient method:
\begin{equation} \nonumber
    \begin{aligned}
&\nabla_\theta\mathop{\mathbb{E}}\limits_{x\sim\mathcal{D},y\sim\pi_\theta(y|x)}R(x,y)\sum_{i=1}^k (1-w_\theta)^{i-1}\\
&\approx \frac{1}{|\mathcal{D}_b|}\sum_{x\sim \mathcal{D}_b}\frac{1}{m}\sum_{j=1}^m R(x,y^j)\sum_{i=1}^k \nabla_\theta[(1-w_\theta)^{i-1}\cdot \pi_\theta(y^j|x)]
    \end{aligned}
\end{equation}

The difficulty arises in computing $\nabla_\theta[(1-w_\theta)^{i-1}\cdot \pi_\theta(y^j|x)]$ as it produces a term of $(i-1)\pi_\theta(y^j|x)(1-w_\theta)^{i-2}\nabla_\theta w_\theta$. This introduces a coupling between the optimization of a particular response \( y \) and all other responses \( y' \) for which \( R(x, y') > 0 \). Such nested gradient dependencies across responses are likely a consequence of the underlying "hitchhiking" behavior inherent in the Max@k objective, and they pose significant challenges for effective optimization.

To address this issue, we observe that although the gradients are interdependent, the additional term $(i-1)\pi_\theta(y^j|x)(1-w_\theta)^{i-2}\nabla_\theta w_\theta$ consistently involves the factor \( \nabla_\theta w_\theta \). This structural pattern suggests that these terms may be globally aggregated or canceled, thereby simplifying the optimization process:
\begin{equation} \nonumber
    \begin{aligned}
&\nabla_\theta\mathop{\mathbb{E}}\limits_{x\sim\mathcal{D},y\sim\pi_\theta(y|x)}R(x,y)\sum_{i=1}^k (1-w_\theta)^{i-1}\\
&= \mathop{\mathbb{E}}\limits_{x\sim\mathcal{D}}\sum_{y\in\textbf{w}}\nabla_\theta[\pi_\theta(y|x)\sum_{i=1}^k (1-w_\theta)^{i-1}] \\
&\text{where we denote }\textbf{w}=\{y|R(x,y)=1\}.\\
&= \mathop{\mathbb{E}}\limits_{x\sim\mathcal{D}}\sum_{y\in\textbf{w}}\nabla_\theta[\pi_\theta(y|x)\frac{1-(1-w_\theta)^k}{w_\theta}] \\
&\text{because of geometric sum formula.}\\
&=\mathop{\mathbb{E}}\limits_{x\sim\mathcal{D}}\sum_{y\in\textbf{w}}\large [
\frac{w_\theta\nabla_\theta\pi_\theta-\pi_\theta\nabla_\theta w_\theta}{w_\theta^2}(1-(1-w_\theta)^k) \\
&\quad+ \frac{\pi_\theta}{w_\theta}\cdot k(1-w_\theta)^{k-1} \nabla_\theta w_\theta \large ] \\
&=\mathop{\mathbb{E}}\limits_{x\sim\mathcal{D}} \large [
\frac{w_\theta\sum_{y\in\textbf{w}}\nabla_\theta\pi_\theta-\sum_{y\in\textbf{w}}\pi_\theta\nabla_\theta w_\theta}{w_\theta^2}(1-(1-w_\theta)^k) \\
&\quad+ \frac{\sum_{y\in\textbf{w}}\pi_\theta}{w_\theta}\cdot k(1-w_\theta)^{k-1} \nabla_\theta w_\theta \large ] \\
&\text{the numerator becomes }w_\theta\nabla_\theta w_\theta-w_\theta\nabla_\theta w_\theta=0\text{, yielding}\\
    \end{aligned}
\end{equation}
\begin{equation}
    \begin{aligned}
&=\mathop{\mathbb{E}}\limits_{x\sim\mathcal{D}} k(1-w_\theta)^{k-1} \nabla_\theta w_\theta=\mathop{\mathbb{E}}\limits_{x\sim\mathcal{D}}\sum_{y\in\textbf{w}}k(1-w_\theta)^{k-1} \nabla_\theta\pi_\theta(y|x) \\
&=\mathop{\mathbb{E}}\limits_{x\sim\mathcal{D},y\sim\pi_\theta(y|x)} k(1-w_\theta)^{k-1}R(x,y) \nabla_\theta\log\pi_\theta(y|x)        
    \end{aligned}
\end{equation}

Surprisingly, despite the underlying complexity of the gradients, the final expression-weighted gradients of the log-likelihood are remarkably clean and interpretable. Specifically, when $R(x, y) = 1$, the weight applied to $\nabla_\theta \log \pi_\theta(y \mid x)$ is $k(1 - w_\theta)^{k - 1}$. Otherwise, when $R(x, y) = 0$, the weight is zero.

We now examine the meaning of the term $k(1 - w_\theta)^{k - 1}$. When $k = 1$, the weight simplifies to $1$, independent of $w_\theta$, which aligns with the standard reinforcement learning objective—equivalent to a Max@1 formulation. For $k > 1$, the weight $k(1 - w_\theta)^{k - 1}$ is a monotonically decreasing function of $w_\theta$, assigning higher importance to smaller values of $w_\theta$ and diminishing the influence of larger ones. This behavior becomes more pronounced as $k$ increases.

The left plot in Figure~\ref{fig:passk} illustrates the decay pattern of $k(1 - w_\theta)^{k - 1}$ as $w_\theta$ increases. The right plot shows the regime in which RSPO amplifies gradients compared to the conventional RL objective, i.e., when $k(1 - w_\theta)^{k - 1} > 1$.

We use the concept of \textit{opportunity cost} to explain why the weight $k(1 - w_\theta)^{k - 1}$ decreases with increasing $w_\theta$. Since the total generation probability sums to 1, increasing the probability of one token/response/pattern necessarily decreases the probability assigned to others. Consequently, if the policy is already capable of generating a correct answer within $k$ attempts (e.g., when $w_\theta = 0.5$ and $k = 10$, the probability of generating only incorrect answers is less than $\frac{1}{1000}$), further reinforcement of its win rate becomes redundant. Instead, it is more beneficial to allocate the remaining probability mass to other regions that may yield higher returns.

We then derive estimators for the Pass@k objective, sampling $n$ responses from the policy $\pi_\theta$ for each prompt.

One obvious estimator is
\begin{equation}
    \frac{1}{|\mathcal{D}_b|}\sum_{x\sim \mathcal{D}_b}
    k(1-\frac{1}{n}\sum_{i=1}^nR(x,y^i))^{k-1}
    \frac{1}{n}\sum_{i=1}^n 
    R(x,y^i)  \nabla_\theta\log\pi_\theta(y^i|x)
\end{equation}
where we estimate $(1-w_\theta)^{k-1}$ by $(1-\frac{1}{n}\sum_{i=1}^nR(x,y^i))^{k-1}$.

The advantage of this estimator is that it is both simple and scalable. It leverages $n$ samples to estimate $w_\theta$, independent of the value of $k$. However, the estimator is biased, as the expectation $E((1-\frac{1}{n}\sum_{i=1}^nR(x,y^i))^{k-1})\neq (1-w_\theta)^{k-1}$.

We then develop an unbiased estimator:
\begin{theorem} \label{theorem:passk}
When $n\ge k$, an unbiased estimator for Pass@k is:
\begin{equation}
    J_\text{pass@k}^\text{RSPO}(\theta)=
    \frac{1}{|\mathcal{D}_b|}\sum_{x\sim \mathcal{D}_b}
    k\frac{\tbinom{n-c}{k-1}}{\tbinom{n-1}{k-1}}
    \frac{1}{n}\sum_{i=1}^n 
    R(x,y^i)  \nabla_\theta\log\pi_\theta(y^i|x)    
\end{equation}
where $c=\sum_{i=1}^n R(x,y^i)$, i.e., the number of correct responses in the sample.
\end{theorem}

The objective function $J_\text{pass@k}^\text{RSPO}(\theta)$ incorporates two key mechanisms to ensure an unbiased gradient estimator. First, when estimating the gradient term $k(1 - w_\theta)^{k-1} R(x,y^i) \nabla_\theta \log \pi_\theta(y^i \mid x)$, $J_\text{pass@k}^\text{RSPO}(\theta)$ uses the remaining $n-1$ responses (excluding $y^i$) to estimate the weight $(1 - w_\theta)^{k-1}$. This separation guarantees independence between the term $R(x,y^i)  \nabla_\theta\log\pi_\theta(y^i|x)$ and the estimator of $(1 - w_\theta)^{k-1}$, thereby ensuring that the expectation of their product equals the product of their expectation ($\mathbb{E}[XY] = \mathbb{E}[X] \mathbb{E}[Y]$). 
Second, to construct an unbiased estimator of $(1 - w_\theta)^{k-1}$ using the $n-1$ responses, $J_\text{pass@k}^\text{RSPO}(\theta)$ enumerates all $\tbinom{n-1}{k-1}$ subsets of size $k-1$ and counts the number of subsets, denoted $\tbinom{n-c}{k-1}$, that consist solely of zero-reward responses. The expectation of each subset matches $(1 - w_\theta)^{k-1}$, yielding an unbiased estimator.

The complete proof of Theorem~\ref{theorem:passk} is provided in Appendix~\ref{app:proof_passk}.

\subsection{RSPO for Max@k} \label{sec:maxk}

In practice, the Max@k objective often appears alongside a continuous reward model, where the probability $P_{\leq,\theta}(y)$ is approximated by $P_{<,\theta}(y) + \pi(y \mid x)$. Notably, in large language models, the policy probability $\pi(y \mid x)$ tends to be extremely small. Given these characteristics, we treat $P_{\leq,\theta}(y)$ and $P_{<,\theta}(y)$ as approximately equal in this section to improve clarity and simplify exposition. For completeness, we provide a rigorous derivation of unbiased estimators without any such approximation in the Appendix~\ref{app:maxk}.

\begin{algorithm}[t]
\caption{Risk-Seeking Policy Optimization}
\label{alg:rspo}
\begin{algorithmic}[1]
\Require Initial policy $\pi_{\theta}$, prompts $\mathcal{D}$, hyperparameters $k,n$
\For{step $= 1,2,...$}
    \State Sample a batch $\mathcal{D}_b$ from $\mathcal{D}$
    \State Sample $n$ responses $\{y^i\}_{i=1}^{n} \sim \pi_\theta(\cdot|x)$ for each $x \in \mathcal{D}_b$
    \State Compute rewards $\{R(x,y^i)\}_{i=1}^{n}$ for every $y^i$ and $x$
    \State Update policy $\pi_{\theta}$ by maximizing $J_\text{pass@k}^\text{RSPO}(\theta)$ or $J_\text{max@k}^\text{RSPO}(\theta)$ 
\EndFor
\State \Return $\pi_{\theta}$
\end{algorithmic}
\end{algorithm}

When we treat $P_{\leq,\theta}(y)$ and $P_{<,\theta}(y)$ as approximately equal, the original Max@k objective
    \begin{equation} \nonumber
\mathop{\mathbb{E}}\limits_{x\sim\mathcal{D},y\sim\pi_\theta(y|x)}R(x,y)\sum_{i=1}^k P_{<,\theta}(y)^{i-1} P_{\leq,\theta}(y)^{k-i}
    \end{equation}
becomes
    \begin{equation}
\mathop{\mathbb{E}}\limits_{x\sim\mathcal{D},y\sim\pi_\theta(y|x)}R(x,y)\cdot k P_{\leq,\theta}(y)^{k-1}
    \end{equation}

We then develop the gradient of the new objective:
    \begin{equation} \nonumber
    \begin{aligned}
&\nabla_\theta\mathop{\mathbb{E}}\limits_{x\sim\mathcal{D},y\sim\pi_\theta(y|x)}R(x,y)\cdot k P_{\leq,\theta}(y)^{k-1} \\
=& \mathop{\mathbb{E}}\limits_{x\sim\mathcal{D}}\sum_{y}R(x,y)k\nabla_\theta[\pi_\theta(y|x)P_{\leq,\theta}(y)^{k-1}] \\
=& \mathop{\mathbb{E}}\limits_{x\sim\mathcal{D}}\sum_{y}R(x,y)k
[P_{\leq,\theta}(y)^{k-1}\nabla_\theta\pi_\theta(y|x)\\
&+(k-1)\pi_\theta(y|x)P_{\leq,\theta}(y)^{k-2}\nabla_\theta P_{\leq,\theta}(y)] \\
    \end{aligned}
    \end{equation}
The equations above imply that computing the gradient with respect to $y$ also induces contributions to the gradients of all $y'$ such that $R(x, y') < R(x, y)$. In other words, the gradient of $y$ is influenced not only by its own value but also by all $y'$ for which $R(x, y') > R(x, y)$. Rearranging the terms accordingly yields:
    \begin{equation} \nonumber
    \begin{aligned}
=& \mathop{\mathbb{E}}\limits_{x\sim\mathcal{D}}\sum_{y}[kR(x,y)
P_{\leq,\theta}(y)^{k-1}\\
&+k(k-1)\sum_{y'\geq y}R(x,y')\pi_\theta(y'|x)P_{\leq,\theta}(y')^{k-2}]\nabla_\theta\pi_\theta(y|x) \\
=& \mathop{\mathbb{E}}\limits_{x\sim\mathcal{D}}\sum_{y}kR(x,y)
P_{\leq,\theta}(y)^{k-1}\nabla_\theta\pi_\theta(y|x)\\
&+\mathop{\mathbb{E}}\limits_{x\sim\mathcal{D}}\sum_{y}k(k-1)\sum_{y'}R(x,y')\pi_\theta(y'|x)P_{\leq,\theta}(y')^{k-2}\nabla_\theta\pi_\theta(y|x) \\
&-\mathop{\mathbb{E}}\limits_{x\sim\mathcal{D}}\sum_{y}k(k-1)\sum_{y'< y}R(x,y')\pi_\theta(y'|x)P_{\leq,\theta}(y')^{k-2}\nabla_\theta\pi_\theta(y|x) \\
=& \mathop{\mathbb{E}}\limits_{x\sim\mathcal{D}}\sum_{y}kR(x,y)
P_{\leq,\theta}(y)^{k-1}\nabla_\theta\pi_\theta(y|x)\\
&-\mathop{\mathbb{E}}\limits_{x\sim\mathcal{D}}\sum_{y}k(k-1)\sum_{y'< y}R(x,y')\pi_\theta(y'|x)P_{\leq,\theta}(y')^{k-2}\nabla_\theta\pi_\theta(y|x) \\
    \end{aligned}
    \end{equation}
because $\sum_{y'}R(x,y')\pi_\theta(y'|x)P_{\leq,\theta}(y')^{k-2}$ is an constant and 

\noindent $\mathbb{E}_{y\sim\pi_\theta} [c \nabla_\theta \log \pi_\theta(y|x)] = \nabla_\theta c = 0$ for any constant $c$.

\begin{figure}[t]
\centering
\includegraphics[width=0.47\textwidth]{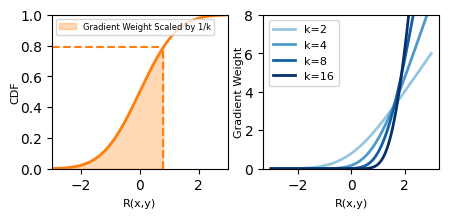}
\caption{ Assume $\pi_\theta$ is an standard Gaussian distribution.
(Left) The $k=2$ case. The curve represents cumulative distribution function of $\pi_\theta$, corresponding to $\sum_{y'< y}\pi_\theta(y'|x)$.
The filled area under the CDF curve thus represents the scaled gradient weight, $\sum_{y'< y}(R(x,y)-R(x,y'))\pi_\theta(y'|x)$; (Right) Gradient weight $k\sum_{y'< y}(R(x,y)-R(x,y'))\pi_\theta(y'|x)(k-1)P_{\leq,\theta}(y')^{k-2}$ as $R(x,y)$ varies; each curve corresponds to a different $k$.}
\label{fig:maxk}
% \vspace{-10mm}
\end{figure}

As a result, the Max@k objective is equivalent to 
\begin{equation}
\begin{aligned}
\mathop{\mathbb{E}}\limits_{x\sim\mathcal{D},y\sim\pi_\theta(y|x)}[kR(x,y)
P_{\leq,\theta}(y)^{k-1}-k(k-1)g(y)]\nabla_\theta\log\pi_\theta(y|x)
\end{aligned}
\end{equation}
where $g(y)=\sum_{y'< y}R(x,y')\pi_\theta(y'|x)P_{\leq,\theta}(y')^{k-2}$.

Although the Max@k objective is more complex than Pass@k— primarily due to the involvement of the function $g(y)$, which depends on other responses—we show that it can still be expressed in the form of a weighted sum of log-likelihood gradients. Specifically, the weight associated with each term is given by $kR(x,y)P_{\leq,\theta}(y)^{k-1}$

\noindent $-k(k-1)g(y)$.

We show that
\begin{equation} \nonumber
    \begin{aligned}
        &kR(x,y)P_{\leq,\theta}(y)^{k-1}-k(k-1)g(y) \\
        \approx& kR(x,y)P_{<,\theta}(y)^{k-1}-k\sum_{y'< y}R(x,y')(k-1)\pi_\theta(y'|x)P_{\leq,\theta}(y')^{k-2} \\
        \approx& kR(x,y)\sum_{y'< y} P(y' \in\{y_i\}_{i=1}^{k-1},R(x,y')=\max_{1\le i\le k-1}R(x,y_i)) \\
        &-k\sum_{y'< y} R(x,y')P(y' \in\{y_i\}_{i=1}^{k-1},R(x,y')=\max_{1\le i\le k-1}R(x,y_i)) \\
        \approx& k\sum_{y'< y}(R(x,y)-R(x,y'))\pi_\theta(y'|x)(k-1)P_{\leq,\theta}(y')^{k-2}
    \end{aligned}
\end{equation}

The gradient weight can be approximated as 
$k \sum_{y' < y} (R(x, y) -$ 

\noindent $R(x, y')) \pi_\theta(y' \mid x) (k - 1) P_{\leq, \theta}(y')^{k - 2}$,
which captures the \textit{marginal contribution} of action $y$ toward increasing the expected maximum reward. This marginal contribution arises from two interacting effects that influence the benefit to the Max@k objective when the probability mass $\nabla_\theta \pi_\theta(y \mid x)$ is increased. First, it reflects the higher likelihood that $y$ itself becomes the top-ranked action among the $k$ sampled draws, represented by the $R(x, y)$ term. Second, it accounts for the decreased probability that any lower-reward action $y' < y$ would have otherwise been selected as the maximum, captured by the $R(x, y')$ term. Together, these effects quantify how shifting probability mass toward $y$ improves the expected outcome.

The left plot in Figure~\ref{fig:maxk} depicts the filled area under the cumulative distribution function of $\pi_\theta(y|x)(k-1)P_{\leq,\theta}(y)^{k-2}$, which corresponds to the gradient weight scaled by $\frac{1}{k}$. The right plot illustrates how the gradient weight increases with different $k$.

The $g(y)$ can be expressed as $\mathop{\mathbb{E}}_{y'}R(x,y')P_{\leq,\theta}(y')^{k-2}\mathbf{1}(R(x,y')$

\noindent$<R(x,y))$, which suggests that $g(y)$ can be approximated via Monte Carlo estimation. A straightforward estimator using $n$ responses first estimates $P_{\leq,\theta}(y)$ by $\hat{P}_{\leq,\theta}(y)=\frac{1}{n}\sum_{j=1}^n \mathbf{1}(R(x,y^j)\leq R(x,y))$ and subsequently estimates $g(y)$ by $\hat{g}(y)=\frac{1}{n}\sum_{j=1}^n \mathbf{1}(R(x,y^j)< R(x,y))R(x,y^j)\hat{P}_{\leq,\theta}(y^j)^{k-2}$, yielding the overall estimator:

\begin{equation}
    \frac{k}{n|\mathcal{D}_b|}\sum_{x\sim \mathcal{D}_b}
    \sum_{i=1}^n 
    [R(x,y^j)\hat{P}_{\leq,\theta}(y^j)^{k-1}-(k-1)\hat{g}(y^j)]
    \nabla_\theta\log\pi_\theta(y^i|x)
\end{equation}

Similar to the straightforward estimator used for Pass@k, this estimator is also simple and scalable, but biased. To address this limitation, we propose an unbiased estimator for Max@k:

\begin{theorem} \label{theorem:maxk}
When $n\ge k$, an unbiased estimator under mild approximations for Max@k is: 
\begin{equation}
\begin{aligned}
J_\text{max@k}^\text{RSPO}(\theta)=&    \frac{k}{n|\mathcal{D}_b|}\sum_{x\sim \mathcal{D}_b} \sum_{i=1}^n 
    [R(x,y^i)\frac{\tbinom{i-1}{k-1}}{\tbinom{n-1}{k-1}}-\\
&\frac{k-1}{n-1}\sum_{j=1}^{i-1}R(x,y^j)\frac{\tbinom{j-1}{k-2}}{\tbinom{n-2}{k-2}}]
    \nabla_\theta\log\pi_\theta(y^i|x)   
\end{aligned}
\end{equation}
where we sort $\{y^i\}_{i=1}^n$ by $R(x,y_1)<R(x,y_2)<...<R(x,y_n)$.
\end{theorem}

The proof of Theorem~\ref{theorem:maxk} closely follows Theorem~\ref{theorem:passk}. When computing the gradient weight for $\nabla_\theta \log \pi_\theta(y^i \mid x)$, the $J_\text{max@k}^\text{RSPO}$ utilizes the remaining $n-1$ responses. Furthermore, to estimate $g(y^i)$, it iterates over all $y^j < y^i$. For each fixed pair $(i, j)$, the $J_\text{max@k}^\text{RSPO}$ then leverages the remaining $n-2$ responses to estimate $P_{\leq,\theta}(y^j)^{k-2}$ within the computation of $g(y^i)$.

The complete proof of Theorem~\ref{theorem:maxk} is provided in Appendix~\ref{app:proof_maxk}.

One key advantage of this estimator lies in the non-negativity of its gradient weights, 
which remain non-negative regardless of the sign of the reward function $R$. Importantly, these weights are exactly zero for $i < k$, and strictly positive for $i \geq k$:
\begin{equation} \nonumber
    \begin{aligned}
        &R(x,y^i)\frac{\tbinom{i-1}{k-1}}{\tbinom{n-1}{k-1}}-
\frac{k-1}{n-1}\sum_{j=1}^{i-1}R(x,y^j)\frac{\tbinom{j-1}{k-2}}{\tbinom{n-2}{k-2}} \\
>&R(x,y^i)[\frac{\tbinom{i-1}{k-1}}{\tbinom{n-1}{k-1}}-
\frac{k-1}{n-1}\sum_{j=1}^{i-1}\frac{\tbinom{j-1}{k-2}}{\tbinom{n-2}{k-2}}] \\
=&\frac{R(x,y^i)}{\tbinom{n-1}{k-1}}[\tbinom{i-1}{k-1}-
\sum_{j=1}^{i-1}\tbinom{j-1}{k-2}]=\frac{R(x,y^i)}{\tbinom{n-1}{k-1}}[\tbinom{i-1}{k-1}-
\sum_{j=k-1}^{i-1}\tbinom{j-1}{k-2}]=0
    \end{aligned}
\end{equation}
where the last step is due to Hockey‐Stick Identity \cite{jones1996generalized}. 

The non-negativity contributes to training stability, as maximizing the log-likelihood with negative weights can cause gradient explosions due to the convex nature of the logarithmic function. As demonstrated in \cite{abdolmaleki2025learning}, negative samples necessitate KL constraints to prevent divergence during training, whereas positive samples do not require such regularization.

\clearpage
\begin{figure*}[t]
\centering
\includegraphics[width=\textwidth]{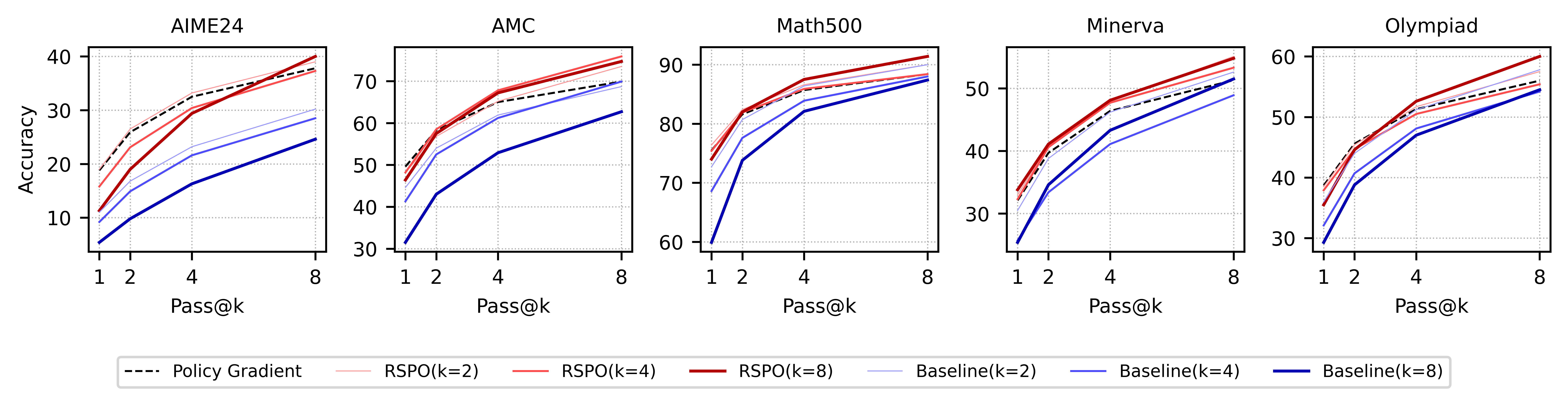}
\vspace*{-7mm}
\caption{Comparison of model accuracy across five math benchmarks at different Pass@k values.}
\label{fig:exp_1}
% \vspace{-5mm}
\end{figure*}

\newpage
\section{Experiments}
\subsection{Setup}
Following the established experimental setting of large language model post-training, we conduct a comprehensive evaluation on math reasoning. Specifically, we post-train Qwen2.5-Math-1.5B \cite{yang2024qwen25mathtechnicalreportmathematical} model on Competition Math dataset \cite{hendrycksmath2021} and assess performance on AIME2024 \cite{li2024numinamath}, AMC \cite{li2024numinamath}, Math500 \cite{hendrycks2021measuring}, Minerva \cite{lewkowycz2022solving}, and OlympiadBench \cite{he2024olympiadbench}. For answer verification, we utilize the xVerify \cite{chen2025xverifyefficientanswerverifier}. 
Our primary focus is on the Pass@k setting, with an additional evaluation of the Max@k setting presented in Section~\ref{sec:exp_maxk}.

We compare RSPO ($k=2,4,8$) with the baseline algorithm ($k=2,4,8$) mentioned in Section~\ref{sec:preliminary} and the policy gradient algorithm, which is equivalent as the special case of RSPO and the baseline when $k=1$. Unless otherwise specified, all experiments generate 16 responses per prompt ($n=16$).

To ensure a fair comparison, we maintain identical experimental settings across algorithms.  For each training step, we sample $128$ prompts from the training set and set the mini-batch size in each step to $256$. We repeat the whole training set for 10 epochs. We set the learning rate to $3e-6$ and set the warm-up ratio to $5\%$. We conduct our experiments using a server with eight 80GB H800 GPU cards. Each experiment takes approximately 16 hours.

We develop our code based on verl \cite{sheng2024hybridflow} framework. We adhere to the Artifact Pledge and guarantee to release the codes upon publication. Here, we note two useful implementation detail. First, we compute the RSPO gradient weight $\tbinom{n-c}{k-1}/\tbinom{n-1}{k-1}$ as $\prod_{i=0}^{k-2}(n - c - i)/(n - 1 - i)$ instead of computing $\tbinom{n-c}{k-1}$ and $\tbinom{n-1}{k-1}$ separately to avoid variable overflow. Second, we remove zero gradient weight responses to reduce training time. This happens when $R(x,y)=0$ for a response or $n-c<k-1$ for all responses in a group.

We provide experiment setting justifications in Appendix~\ref{app:justification}.

\subsection{Main Analyses} \label{sec:exp_main}
Figure~\ref{fig:exp_1} shows the main experiment results. 
Each graph shows the accuracy achieved at various Pass@k values for a given dataset. 
Policy Gradient is shown as a black dashed line. RSPO variants are represented by increasingly darker red lines as $k$ increases, while Baseline variants are shown in shades of blue. In addition, thicker lines also indicate larger $k$ values.

First, we observe that across all datasets the blue curves lie consistently below the red curves, demonstrating that RSPO outperforms the baseline. Moreover, within the baseline family, performance degrades as \(k\) increases: the algorithm with \(k=2\) outperforms \(k=4\), which in turn outperforms \(k=8\). This pattern highlights the "hitchhiking" problem inherent to the baselines: low‐reward responses are reinforced whenever they co‐occur with high‐reward responses in the same group. As \(k\) grows, this issue intensifies, since low‐reward responses have an even greater chance of being boosted alongside their high‐reward counterparts. 

Second, RSPO tends to achieve its best results when the training hyperparameter \(k\) matches the evaluation Pass@k. For instance, on AIME24, the policy‐gradient method, i.e. RSPO (\(k=1\)), attains its highest accuracy at Pass@1, whereas RSPO trained with $k=8$ peaks at Pass@8. This trend also holds for other values of $k$ and across all datasets. This finding suggests that practitioners may benefit from choosing their training hyperparameter to coincide with their intended Pass@k evaluation.

Recent work \cite{yue2025does} argues that RL may not incentivize reasoning capacity in LLMs, observing that base models outperform their RL–fine‑tuned counterparts on Pass@k when \(k\) is large. We contend that this phenomenon arises because their training objective (\(k=1\)) is misaligned with their evaluation metric (\(k>1\)). 
To resolve this, the community should first agree on the most appropriate choice of \(k\) for measuring a model’s reasoning capacity, and only then assess whether RL improves reasoning capacity. 
Our experiments demonstrate that RL—including RSPO—does, in fact, enhance reasoning ability under a suitably chosen Pass@k metric.

Finally, we provide empirical evidence of why RSPO excels on Pass@k. In Section~\ref{sec:passk}, we theoretically show that the gradient weight $k(1-w_\theta)^{k-1}$ naturally curbs over‑exploitation once Pass@k is sufficiently high. Figure~\ref{fig:exp_2} confirms this analysis: as \(k\) increases, RSPO maintains higher entropy, yielding more diverse generations and thereby greater opportunities to satisfy the Pass@k criterion.

\vspace{-2mm}
\begin{figure}[h]
\centering
\includegraphics[width=0.45\textwidth]{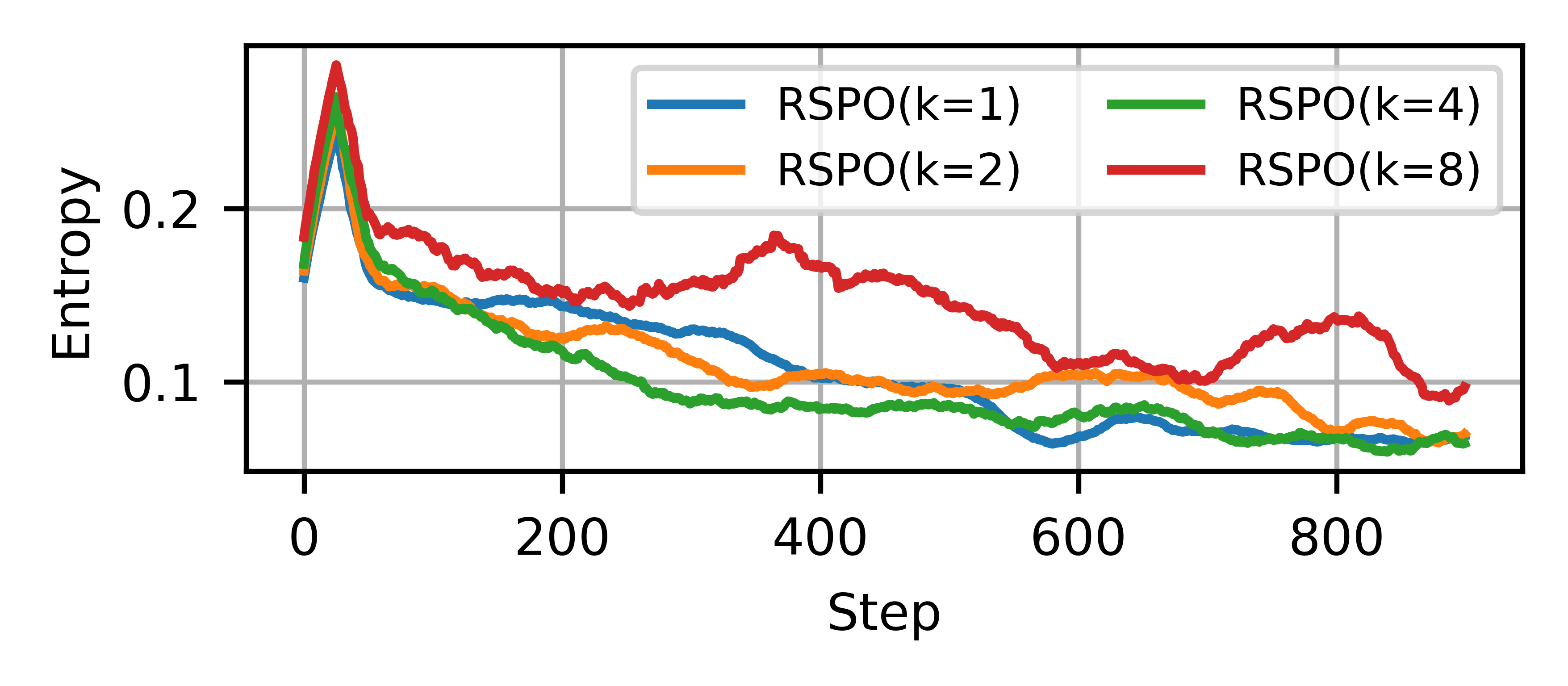}
\vspace*{-6mm}
\caption{RSPO Entropy over Training Steps.}
\label{fig:exp_2}
\end{figure}

\clearpage

\begin{figure*}[t]
\centering
\includegraphics[width=\textwidth]{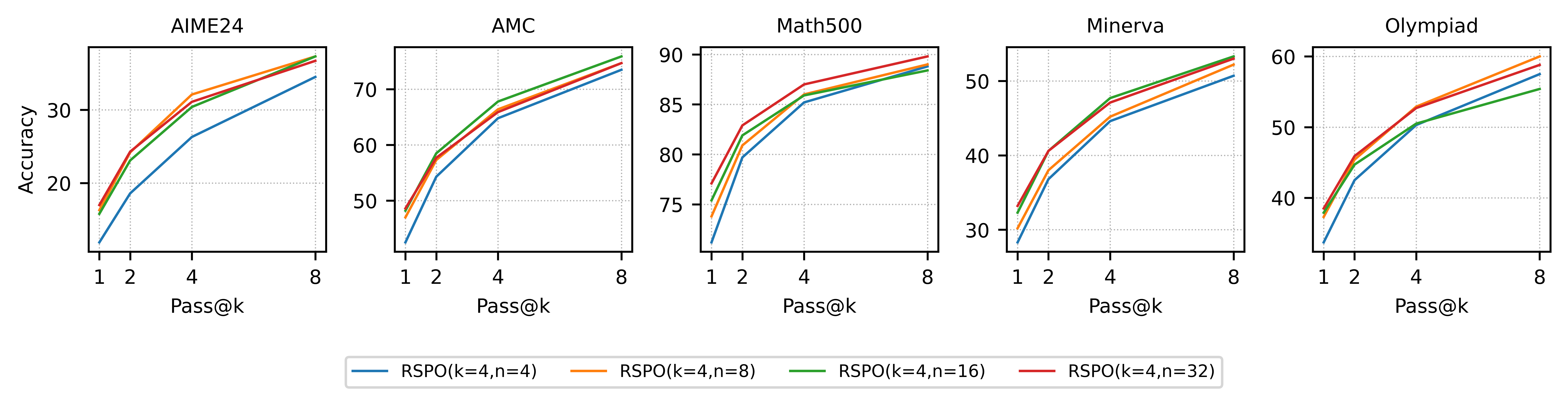}
\vspace*{-7mm}
\caption{Comparison of model accuracy as $n$ varies at different Pass@k values.}
\label{fig:exp_3}
\vspace{-3mm}
\end{figure*}

\begin{figure*}[t]
\centering
\includegraphics[width=\textwidth]{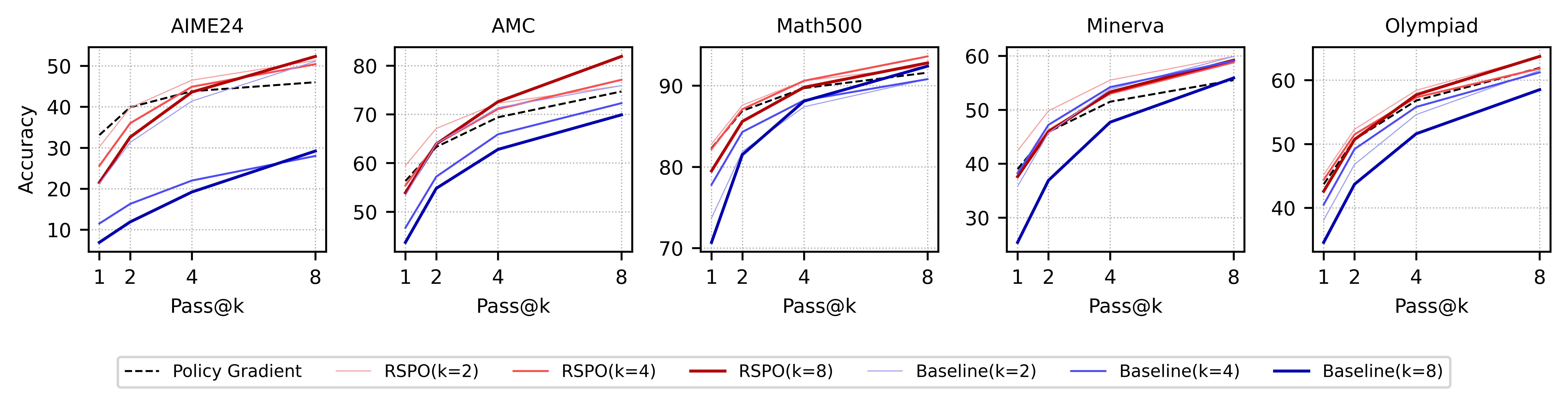}
\vspace*{-7mm}
\caption{Comparison of model accuracy for the 7B model at different Pass@k values.}
\label{fig:exp_4}
\vspace{-3mm}
\end{figure*}

\subsection{Scalability}
In this experiment, we fix \(k = 4\) and vary the number of samples per prompt, \(n \in \{4,8,16,32\}\), to evaluate the scalability of RSPO. Figure~\ref{fig:exp_3} presents the results, showing that increasing \(n\) generally improves accuracy across all Pass@k metrics. In particular, \(n=4\) yields the lowest performance regardless of dataset or metric, while \(n=32\) excels across settings. This phenomenon arises because RSPO’s estimator enumerates all \(\binom{n-1}{k-1}\) subsets: when \(n=k\), only one subset exists, forcing each gradient weight to be either \(k\) or \(0\) with bad flexibility. Consequently, we recommend using at least \(n=2k\), and choosing \(n\) as large as computationally feasible.

\subsection{Robustness}
\subsubsection{Alternative Model Size} 
We further test our algorithm on Qwen2.5-Math-7B, a larger base model than the default setting. Figure~\ref{fig:exp_4} shows that the findings in Section~\ref{sec:exp_main} also hold in the 7B setting, demonstrating the robustness of our algorithm.

\vspace{-4mm}
\begin{figure}[h]
\centering
\includegraphics[width=0.44\textwidth]{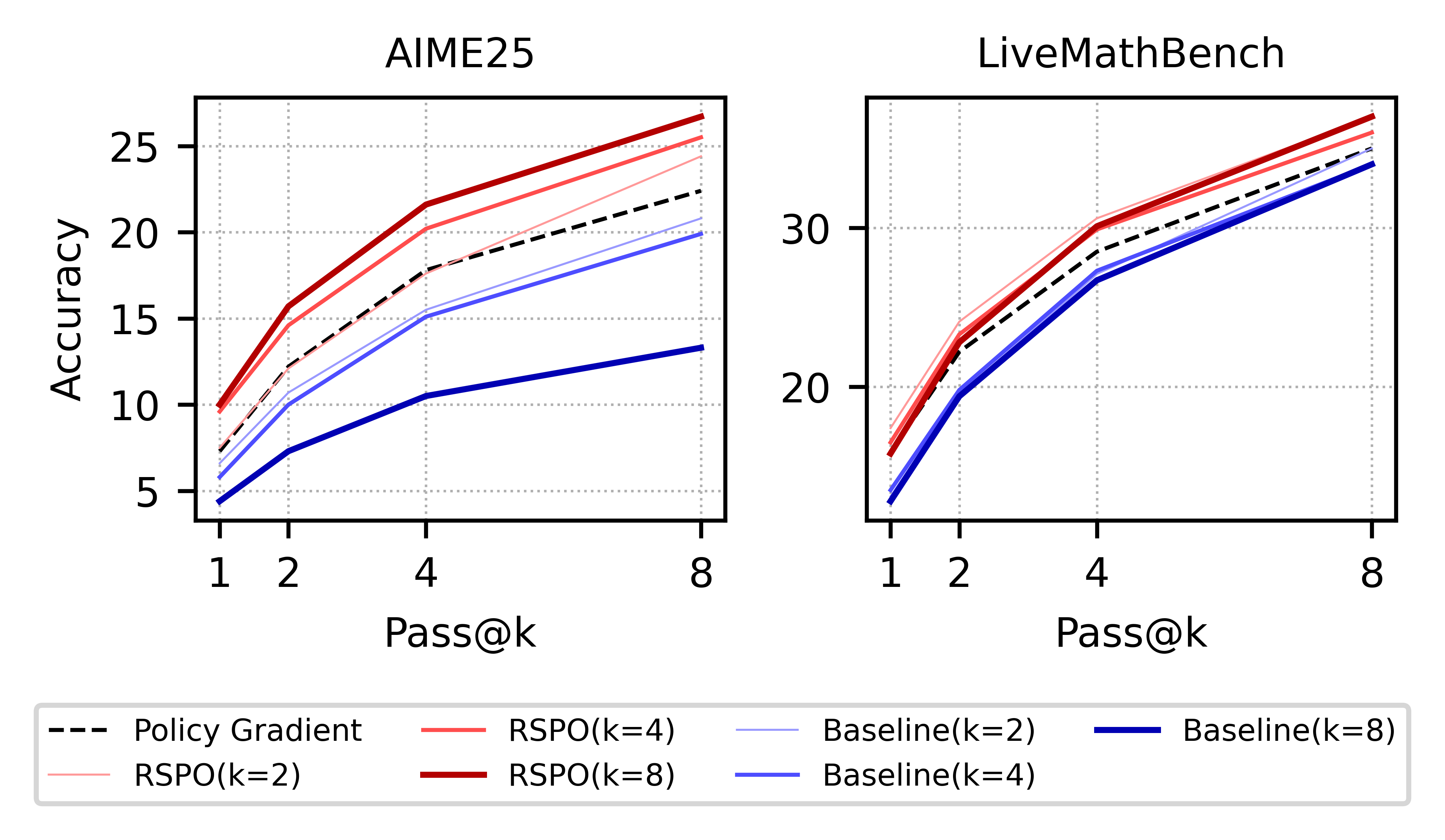}
\vspace*{-5mm}
\caption{Model accuracy on up-to-date datasets.}
\label{fig:exp_5}
\vspace{-3mm}
\end{figure}

\subsubsection{Alternative Datasets}
Recent study \cite{wu2025reasoning} indicates that the AIME2024, AMC, and Math500 datasets were likely included in Qwen2.5’s training data, which could bias experimental outcomes. For example, spurious reward signals may inadvertently boost measured accuracy \cite{shao2025spurious}. To guard against this contamination, we have re‑evaluated our results on two additional datasets—AIME2025 \cite{opencompass_aime2025} and LiveMathBench(202505) \cite{liu2024your}—both of which were created after Qwen2.5 and have been verified to be free of overlap with training corpus \cite{wu2025reasoning}. As shown in Figure~\ref{fig:exp_5}, our experimental results remain robust against potential training data contamination.

\subsection{Max@k Evaluation} \label{sec:exp_maxk}
To evaluate the Max@k metric in the context of math reasoning, we design the following reward: $R(x,y)=1-0.5*len(y)/max\_len$ if $y$ is correct, otherwise $R(x,y)=0$. This formulation offers two advantages. First, it avoids the proxy reward problem \cite{skalse2022defining} commonly encountered in RLHF. Second, the reward function is inherently meaningful, as it incentivizes shorter responses, thus lowering inference costs. The Max@k results, presented in Figure~\ref{fig:exp_6} in Appendix, follow trends similar to those observed with Pass@k, further validating the effectiveness of RSPO under the Max@k metric.

\section{Conclusion}

In this paper, we address the mismatch problem between training and evaluation objectives in LLM post-training. We propose RSPO, a method that decouples individual responses from response sets to avoid inefficiencies arising from co-occurrence. We construct unbiased estimators for both Pass@k and Max@k, and support RSPO with clear theoretical analysis and empirical validation.

% \newpage
\clearpage
\bibliographystyle{ACM-Reference-Format}
\bibliography{reference}

\appendix
\clearpage
\begin{figure*}[t]
\centering
\includegraphics[width=\textwidth]{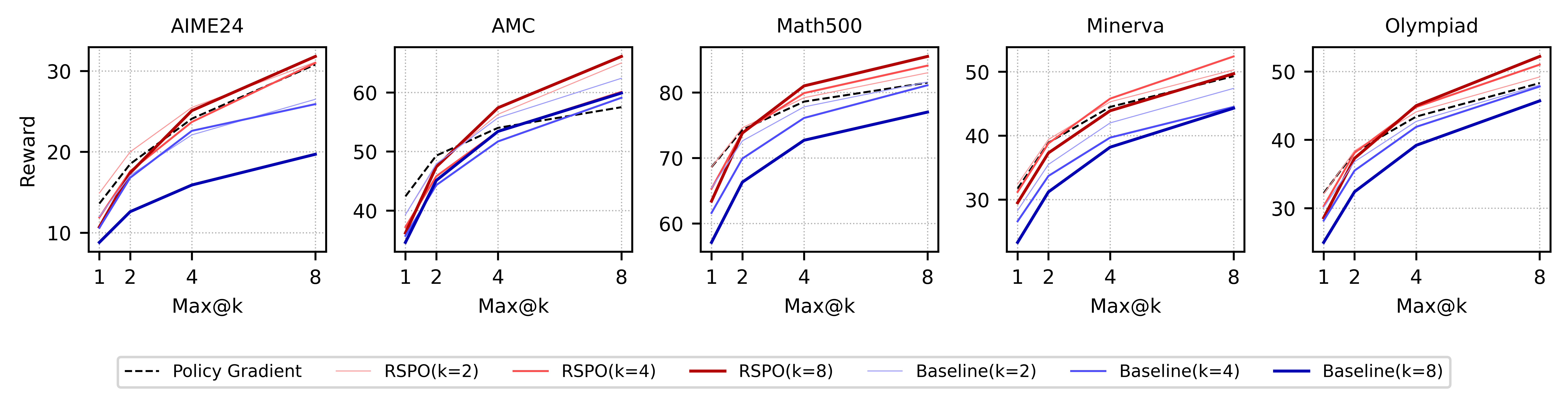}
\vspace*{-7mm}
\caption{Model comparison across five math benchmarks at different Max@k values.}
\label{fig:exp_6}
\vspace{-3mm}
\end{figure*}

\begin{figure*}[t]
\centering
\includegraphics[width=\textwidth]{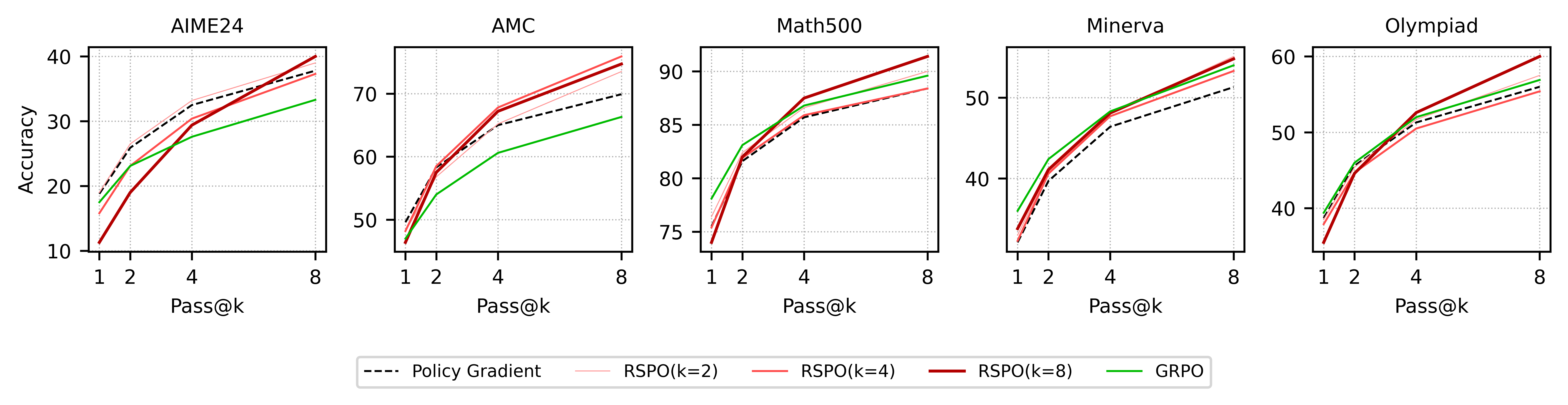}
\vspace*{-7mm}
\caption{Model comparison with GRPO across five math benchmarks at different Pass@k values.}
\label{fig:exp_7}
\vspace{-3mm}
\end{figure*}

\section*{Appendix}
\section{Experiment Setting Justifications} \label{app:justification}

In this appendix, we provide a justification for our experimental setup and address potential concerns that reviewers may raise.

\textbf{1. Why are the experiment figures dense? Why not use tables instead?} 

There are two reasons for this choice. First, tables would occupy significantly more space than figures. A table would likely span an entire page, whereas a figure is more compact. Second, detecting trends in tables can be challenging. For example, the relationship between different values of \( k \) may be difficult to discern. In contrast, figures provide a clearer visual representation of these trends.

\textbf{2. Why don't you compare RSPO with other popular RL algorithms?} 

We justify this experimental setup both theoretically and empirically. From a theoretical perspective, existing RL algorithms for LLM post-training generally optimize the risk-neutral objective \( \mathop{\mathbb{E}}_{x \sim \mathcal{D}, y \sim \pi_\theta(y|x)}R(x,y) \). In contrast, our focus is on the Pass@k and Max@k metrics, where the objective mismatch would lead to suboptimal performance of the existing RL algorithms in evaluation.

Empirically, we tested the performance of the widely-used GRPO \cite{shao2024deepseekmath} in Figure~\ref{fig:exp_7} and observed that it exhibits performance and trends similar to those of the Policy Gradient method. This finding suggests that Policy Gradient serves as a reasonable representative baseline for algorithms that optimize the risk-neutral objective.

\textbf{3. Why don't you test Max@k on alignment/RLHF tasks?} 

LLM post-training tasks generally fall into two categories: reinforcement learning with verifiable rewards (RLVR) and reinforcement learning from human feedback (RLHF). In the RLHF setting, the effectiveness of RL algorithms depends on how accurately the proxy rewards align with the true rewards. This \textit{proxy reward issue} \cite{skalse2022defining} is a well-documented challenge in reinforcement learning. As a result, it is common practice \cite{shao2024deepseekmath, wu2025reasoning, yue2025does} to evaluate post-training algorithms in the RLVR setting, where rewards are verifiable and more directly related to performance.

\section{Proofs}
\subsection{Proof of Theorem~\ref{theorem:passk}} \label{app:proof_passk}
\begin{proof}
    Consider using $y^i$ to estimate 
    \[
    \mathop{\mathbb{E}}\limits_{y\sim\pi_\theta(y|x)}R(x,y) \nabla_\theta\log\pi_\theta(y|x)
    \]
    and $\textbf{y}^{-i}=\{y^j|1\leq j\leq n,j\neq i\}$ to estimate $(1-w_\theta)^{k-1}$.
    
    Obviously, $ R(x,y^i) \nabla_\theta\log\pi_\theta(y^i|x)$ is an unbiased estimator of $\mathop{\mathbb{E}}\limits_{y\sim\pi_\theta(y|x)}R(x,y) \nabla_\theta\log\pi_\theta(y|x)$, because 
    \begin{equation}
        \begin{aligned}
            &\mathop{\mathbb{E}}\limits_{y^i\sim\pi_\theta(y|x)}R(x,y^i) \nabla_\theta\log\pi_\theta(y^i|x) \\
            =&\mathop{\mathbb{E}}\limits_{y\sim\pi_\theta(y|x)}R(x,y) \nabla_\theta\log\pi_\theta(y|x)
        \end{aligned}
    \end{equation}

    Next we prove that $\frac{\tbinom{n-1-c^{-i}}{k-1}}{\tbinom{n-1}{k-1}}$, with $c^{-i}=\sum_{y^j \in \textbf{y}^{-i}} R(x,y^j)$, is an unbiased estimator of $(1-w_\theta)^{k-1}$:

    For a size $k-1$ subset $s$ of $\textbf{y}^{-i}$ ($|s|=k-1,s\subseteq\textbf{y}^{-i}$), we denote 
    \[
I(s)= \begin{cases}
    1 & R(x,y^j)=0,\forall y^j\in s \\
    0& otherwise
\end{cases}
\]
We have
\begin{equation}
    \begin{aligned}
    \mathop{\mathbb{E}}(I(s))=\mathop{\mathbb{E}}(\prod_{y^j\in s}(1-R(x,y^j))^{k-1})=(1-w_\theta)^{k-1}
    \end{aligned}
\end{equation}
Then
\begin{equation}
    \begin{aligned}
        &\mathop{\mathbb{E}}(\frac{\tbinom{n-1-c^{-i}}{k-1}}{\tbinom{n-1}{k-1}})
        =\frac{1}{\tbinom{n-1}{k-1}}\mathop{\mathbb{E}}(\sum_{|s|=k-1,s\subseteq\textbf{y}^{-i}}I(s)) \\
        =&\frac{1}{\tbinom{n-1}{k-1}}\sum_{|s|=k-1,s\subseteq\textbf{y}^{-i}}\mathop{\mathbb{E}}(I(s))=\frac{1}{\tbinom{n-1}{k-1}}\sum_{|s|=k-1,s\subseteq\textbf{y}^{-i}}(1-w_\theta)^{k-1} \\
        =&\frac{1}{\tbinom{n-1}{k-1}}\tbinom{n-1}{k-1}(1-w_\theta)^{k-1}=(1-w_\theta)^{k-1}
    \end{aligned}
\end{equation}

As ${y^i}$ are independent of $\textbf{y}^{-i}$, 
\begin{equation}
    \begin{aligned}
        &\mathop{\mathbb{E}}(\frac{\tbinom{n-1-c^{-i}}{k-1}}{\tbinom{n-1}{k-1}}\cdot R(x,y^i) \nabla_\theta\log\pi_\theta(y^i|x)) \\
        =&\mathop{\mathbb{E}}(\frac{\tbinom{n-1-c^{-i}}{k-1}}{\tbinom{n-1}{k-1}})\cdot \mathop{\mathbb{E}}(R(x,y^i) \nabla_\theta\log\pi_\theta(y^i|x)) \\
        =&(1-w_\theta)^{k-1} \mathop{\mathbb{E}}(R(x,y^i) \nabla_\theta\log\pi_\theta(y^i|x) \\
        =&\mathop{\mathbb{E}}((1-w_\theta)^{k-1}R(x,y^i) \nabla_\theta\log\pi_\theta(y^i|x)
    \end{aligned}
\end{equation}
As a result, $\frac{\tbinom{n-1-c^{-i}}{k-1}}{\tbinom{n-1}{k-1}}\cdot R(x,y^i) \nabla_\theta\log\pi_\theta(y^i|x)$ is an unbiased estimator of original Pass@k objective.

Because the above proof holds for every $i$,
\begin{equation}
    \begin{aligned}
        &\frac{1}{n}\sum_{i=1}^n \frac{\tbinom{n-1-c^{-i}}{k-1}}{\tbinom{n-1}{k-1}}\cdot R(x,y^i) \nabla_\theta\log\pi_\theta(y^i|x) \\
        =&\frac{\tbinom{n-c}{k-1}}{\tbinom{n-1}{k-1}}
    \frac{1}{n}\sum_{i=1}^n 
    R(x,y^i)  \nabla_\theta\log\pi_\theta(y^i|x)    
    \end{aligned}
\end{equation} 
is also an unbiased estimator. The theorem has been proofed. \qedhere

\end{proof}
\subsection{Proof of Theorem~\ref{theorem:maxk}} \label{app:proof_maxk}
\begin{proof}
For a subset $s$ of $\{y^i\}_{i=1}^n$, we denote
    \[
I(s|y)= \begin{cases}
    1 & R(x,y')\leq R(x,y),\forall y'\in s \\
    0& otherwise
\end{cases}
\]

We first prove:
\begin{equation}
    \begin{aligned}
        &\mathbb{E}[\frac{\binom{i-1}{k-1}}{\binom{n-1}{k-1}}|y^i] 
        =\mathbb{E}[\frac{1}{\binom{n-1}{k-1}}\sum_{\substack{|s|=k-1,\\s\subseteq\{y^j|1\leq j\leq n,j\neq i\}}}I(s|y^i)|y^i] \\
        =& \frac{1}{\binom{n-1}{k-1}}\sum_{\substack{|s|=k-1,\\s\subseteq\{y^j|1\leq j\leq n,j\neq i\}}}\mathbb{E}[I(s|y^i)|y^i] \\
        =& \frac{1}{\binom{n-1}{k-1}}\sum_{\substack{|s|=k-1,\\s\subseteq\{y^j|1\leq j\leq n,j\neq i\}}} P_{\leq,\theta}(y^i)^{k-1} \\
        =& \frac{\binom{n-1}{k-1}}{\binom{n-1}{k-1}}P_{\leq,\theta}(y^i)^{k-1}=P_{\leq,\theta}(y^i)^{k-1}
    \end{aligned}
\end{equation}
Multiplying $R(x,y^i)\nabla_\theta\log\pi_\theta(y^i|x)$ to both sides,
\begin{equation}
    \begin{aligned}
        &R(x,y^i)P_{\leq,\theta}(y^i)^{k-1}\nabla_\theta\log\pi_\theta(y^i|x) \\
        =&R(x,y^i)\mathbb{E}[\frac{\binom{i-1}{k-1}}{\binom{n-1}{k-1}}|y^i]  \nabla_\theta\log\pi_\theta(y^i|x) \\
        =&\mathbb{E}[R(x,y^i)\frac{\binom{i-1}{k-1}}{\binom{n-1}{k-1}}\nabla_\theta\log\pi_\theta(y^i|x)|y^i]   
    \end{aligned}
\end{equation}
Taking expectation to both sides yield:
\begin{equation}
    \begin{aligned}
        &\mathbb{E}[R(x,y)P_{\leq,\theta}(y)^{k-1}\nabla_\theta\log\pi_\theta(y|x)] \\
                =&\mathbb{E}[R(x,y^i)P_{\leq,\theta}(y^i)^{k-1}\nabla_\theta\log\pi_\theta(y^i|x)] \\
        =&\mathbb{E}[\mathbb{E}[R(x,y^i)\frac{\binom{i-1}{k-1}}{\binom{n-1}{k-1}}\nabla_\theta\log\pi_\theta(y^i|x)|y^i] ] \\
        =&\mathbb{E}[R(x,y^i)\frac{\binom{i-1}{k-1}}{\binom{n-1}{k-1}}\nabla_\theta\log\pi_\theta(y^i|x)]
    \end{aligned}
\end{equation}
Because the above proof holds for every $i$,
\begin{equation}
    \begin{aligned}
    &\mathbb{E}[\frac{1}{n}\sum_{i=1}^nR(x,y^i)\frac{\binom{i-1}{k-1}}{\binom{n-1}{k-1}}\nabla_\theta\log\pi_\theta(y^i|x)] \\
    =&\frac{1}{n}\sum_{i=1}^n\mathbb{E}[R(x,y^i)\frac{\binom{i-1}{k-1}}{\binom{n-1}{k-1}}\nabla_\theta\log\pi_\theta(y^i|x)] \\
    =&\mathbb{E}[R(x,y)P_{\leq,\theta}(y)^{k-1}\nabla_\theta\log\pi_\theta(y|x)]
    \end{aligned}
\end{equation}

Similarly, we can proof
\begin{equation}
    \begin{aligned}
    &\mathbb{E}[\frac{1}{n-1}\sum_{j=1}^{i-1}R(x,y^j)\frac{\tbinom{j-1}{k-2}}{\tbinom{n-2}{k-2}}|y^i] \\
    =&\mathbb{E}[\frac{1}{n-1}[\sum_{j=1}^{i-1}R(x,y^j)\frac{\tbinom{j-1}{k-2}}{\tbinom{n-2}{k-2}}+\sum_{j=i+1}^{n}0]|y^i]  \\
    =&\mathbb{E}[\sum_{y'}\pi_\theta(y'|x)R(x,y')P_{\leq,\theta}(y')^{k-2}\mathbf{1}(R(x,y')<R(x,y^i))|y^i]\\
    =&\mathbb{E}[\sum_{y'< y^i}R(x,y')\pi_\theta(y'|x)P_{\leq,\theta}(y')^{k-2}|y^i]
    =\mathbb{E}[g(y^i)|y^i]
    \end{aligned}
\end{equation}
in which $\mathbb{E}[\frac{\binom{j-1}{k-2}}{\binom{n-2}{k-2}}|y^i,y^j]=\mathbb{E}[P_{\leq,\theta}(y^j)^{k-2}|y^i]$.

As a result,
\begin{equation}
    \begin{aligned}
    &\mathbb{E}[\frac{1}{n}\sum_{i=1}^n\frac{k-1}{n-1}\sum_{j=1}^{i-1}R(x,y^j)\frac{\tbinom{j-1}{k-2}}{\tbinom{n-2}{k-2}}
    \nabla_\theta\log\pi_\theta(y^i|x)] \\
    =&\mathbb{E}[(k-1)g(y)\nabla_\theta\log\pi_\theta(y|x)]
    \end{aligned}
\end{equation}

Since both parts are unbiased, the theorem has been proofed.

\end{proof}

\newpage

\section{RSPO for Max@k without approximation} \label{app:maxk}
In Section~\ref{sec:maxk}, we approximated $P_{\leq,\theta}(y)$ and $P_{<,\theta}(y)$ equally for simplicity and clearance. This approximation is suitable for reward models which generate continues rewards. However, when reward models generate discrete rewards-for example, generative rewards models-such approximation may lead to significant misalignment.

In this section, we develop unbiased estimators for Max@k without any approximation. We start from the original Max@k object:
    \begin{equation} \nonumber
    \begin{aligned}
&\nabla_\theta\mathop{\mathbb{E}}\limits_{x\sim\mathcal{D},y\sim\pi_\theta(y|x)}R(x,y)\sum_{t=1}^k P_{<,\theta}(y)^{t-1} P_{\leq,\theta}(y)^{k-t} \\
=& \mathop{\mathbb{E}}\limits_{x\sim\mathcal{D}}\sum_{t=1}^k\sum_{y}R(x,y)\nabla_\theta[\pi_\theta(y|x)P_{<,\theta}(y)^{t-1} P_{\leq,\theta}(y)^{k-t}] \\
=&
\mathop{\mathbb{E}}\limits_{x\sim\mathcal{D}}\sum_{t=1}^k\sum_{y}R(x,y)P_{<,\theta}(y)^{t-1} P_{\leq,\theta}(y)^{k-t}\nabla_\theta\pi_\theta(y|x) \\
+&\mathop{\mathbb{E}}\limits_{x\sim\mathcal{D}}\sum_{t=1}^k\sum_{y}R(x,y)\pi_\theta(y|x)(t-1)P_{<,\theta}(y)^{t-2} P_{\leq,\theta}(y)^{k-t}\nabla_\theta P_{<,\theta}(y) \\
+&\mathop{\mathbb{E}}\limits_{x\sim\mathcal{D}}\sum_{t=1}^k\sum_{y}R(x,y)\pi_\theta(y|x)P_{<,\theta}(y)^{t-1} (k-t)P_{\leq,\theta}(y)^{k-t-1}\nabla_\theta P_{\leq,\theta}(y) \\
&\text{gather terms related to }\nabla_\theta\pi_\theta(y|x) \text{ together}:\\
=&
\mathop{\mathbb{E}}\limits_{x\sim\mathcal{D}}\sum_{t=1}^k\sum_{y}[R(x,y)P_{<,\theta}(y)^{t-1} P_{\leq,\theta}(y)^{k-t} \\
&+\sum_{y'> y}R(x,y')\pi_\theta(y'|x)(t-1)P_{<,\theta}(y')^{t-2} P_{\leq,\theta}(y')^{k-t} \\
&+\sum_{y'\geq y}R(x,y')\pi_\theta(y'|x)P_{<,\theta}(y')^{t-1} (k-t)P_{\leq,\theta}(y')^{k-t-1}]\nabla_\theta\pi_\theta(y|x) \\
&\text{because }\mathbb{E}_{y\sim\pi_\theta} [c \nabla_\theta \log \pi_\theta(y|x)] = \nabla_\theta\text{ c = 0 for any constant c:}\\
=&
\mathop{\mathbb{E}}\limits_{x\sim\mathcal{D}}\sum_{t=1}^k\sum_{y}[R(x,y)P_{<,\theta}(y)^{t-1} P_{\leq,\theta}(y)^{k-t} \\
&-\sum_{y'\leq y}R(x,y')\pi_\theta(y'|x)(t-1)P_{<,\theta}(y')^{t-2} P_{\leq,\theta}(y')^{k-t} \\
&-\sum_{y'< y}R(x,y')\pi_\theta(y'|x)P_{<,\theta}(y')^{t-1} (k-t)P_{\leq,\theta}(y')^{k-t-1}]\nabla_\theta\pi_\theta(y|x) \\
&\text{rearrange terms to remove }t-1,k-t \text{ for }y'<y:\\
=&
\mathop{\mathbb{E}}\limits_{x\sim\mathcal{D}}\sum_{y}[R(x,y)\sum_{t=1}^k P_{<,\theta}(y)^{t-1} P_{\leq,\theta}(y)^{k-t} \\
&-\sum_{y'=y}R(x,y')\pi_\theta(y'|x)\sum_{t=1}^{k-1}t P_{<,\theta}(y')^{t-1} P_{\leq,\theta}(y')^{k-t-1} \\
&-\sum_{y'< y}R(x,y')\pi_\theta(y'|x)\sum_{t=1}^{k-1}k P_{<,\theta}(y')^{t-1}P_{\leq,\theta}(y')^{k-t-1}]\nabla_\theta\pi_\theta(y|x) \\
    \end{aligned}
    \end{equation}

Note that $y'=y$ indicates $R(x,y')=R(x,y)$ according to our definition. In other words, the second term iterate all the responses with the equivalent reward as $y$.

\newpage

Similar as Section~\ref{sec:maxk}, we will use $n-1$ responses excluding $y$ to estimate $P_{<,\theta}(y)^{t-1} P_{\leq,\theta}(y)^{k-t}$ and $n-2$ responses excluding $y,y'$ to estimate $P_{<,\theta}(y')^{t-1}P_{\leq,\theta}(y')^{k-t-1}$. We show that
\begin{lemma}
    Given $n_0$ sampled responses, in which $c_<$ responses have less rewards than $y_0$ and $c_=$ responses have equivalent rewards as $y_0$, an unbiased estimator of $P_{<,\theta}(y_0)^a P_{\leq,\theta}(y_0)^b$ is $\frac{\tbinom{c_<}{a}\tbinom{c_<+c_=-a}{b}}{\tbinom{n_0}{a+b}\tbinom{a+b}{a}}$ when $n_0\geq a+b$.
\end{lemma}

Using this lemma, we are able to compute an unbiased estimator for the Max@k object by Algorithm~\ref{alg:unbiased-grad}.

\begin{algorithm}[ht]
\caption{An Unbiased Estimator for Max@k}\label{alg:unbiased-grad}
\begin{algorithmic}[1]
\Require prompt $x$, number of samples $n$, Max@k size $k$, policy $\pi_\theta$, reward function $R$
\State Sample $n$ responses from the policy:
\[
  \{y_1,\dots,y_n\}\;\overset{\mathrm{iid}}{\sim}\;\pi_\theta(\cdot\mid x)
\]
\State initialize gradient estimator $\hat g \leftarrow 0$
\For{$i = 1,\dots,n$}                     \Comment{Loop over main sample $y_i$}
  \State $y \gets y_i,\quad r \gets R(x,y)$
  \State let
  \[
    c_< \gets \bigl|\{j\neq i: R(x,y_j)<r\}\bigr|,\quad
    c_= \gets \bigl|\{j\neq i: R(x,y_j)=r\}\bigr|
  \]
  \For{$t = 1,\dots,k$}                  \Comment{First term estimation}
    \State 
    \[
      \hat E^{(1)}_{i,t}
      \;=\;
      \frac{\displaystyle\binom{c_<}{t-1}\,\binom{c_<+c_=- (t-1)}{k-t}}
           {\displaystyle\binom{n-1}{k-1}\,\binom{k-1}{t-1}}
    \]
    \State
    \(\hat g \;+\!=\; r\;\hat E^{(1)}_{i,t}\;\nabla_\theta\log\pi_\theta(y\mid x)\)
  \EndFor
  \For{each $j\neq i$}                   \Comment{Second and third term estimation}
    \State $y'\gets y_j,\quad r' \gets R(x,y')$
    \State let
    \[
      c'_< \gets \bigl|\{\ell\neq i,j : R(x,y_\ell)<r'\}\bigr|,\quad
      c'_=\gets \bigl|\{\ell\neq i,j : R(x,y_\ell)=r'\}\bigr|
    \]
    \For{$t = 1,\dots,k-1$}
        \[
        \hat E^{(2,3)}_{ij,t}
        =\;
        \frac{\displaystyle\binom{c'_<}{t-1}\,\binom{c'_<+c'_=- (t-1)}{k-t-1}}
             {\displaystyle\binom{n-2}{k-2}\,\binom{k-2}{t-1}}
        \]
      \If{$r'=r$}                        \Comment{Second term}
      \State
      \(\hat g \;-\!=\;
         r'\,\pi_\theta(y'\mid x)\,
         \frac{t}{n-1}\hat E^{(2,3)}_{ij,t}\,
         \nabla_\theta\log\pi_\theta(y\mid x)\)
      \ElsIf{$r'<r$}                    \Comment{Third term}
      \State
      \(\hat g \;-\!=\;
         r'\,\pi_\theta(y'\mid x)\,
         \frac{k}{n-1}\hat E^{(2,3)}_{ij,t}\,
         \nabla_\theta\log\pi_\theta(y\mid x)\)
      \EndIf
    \EndFor
  \EndFor
\EndFor
\State \Return \(\hat g / n\)
\end{algorithmic}
\end{algorithm}

\newpage

We then derive a closed-form unbiased estimator without summing over $t$, which simplifies computation and enables comparisons with other estimators.

First, we show that

\begin{lemma} \label{lemma:2}
    \begin{equation}
        \sum_{0\leq a\leq m,b=m-a} \frac{\tbinom{c_<}{a}\tbinom{c_<+c_=-a}{b}}{\tbinom{a+b}{a}}=\frac{m+1}{c_=+1}\Large[\tbinom{c_<+c_=+1}{m+1}-\tbinom{c_<}{m+1}\Large]
    \end{equation}
\end{lemma}

Using this lemma, we can derive an unbiased estimator 
\begin{equation}
    \frac{1}{\tbinom{n-2}{k-2}}\frac{k-1}{c'_=+1}[\tbinom{c'_<+c'_=+1}{k-1}-\tbinom{c'_<}{k-1}]
\end{equation}
for $\sum_{t=1}^{k-1} P_{<,\theta}(y')^{t-1}P_{\leq,\theta}(y')^{k-t-1}$. The $c'_<$ and $c'_=$ are the number of responses with rewards less than or equal to $y'$ among the $n-2$ responses, respectively.

Moreover, as the rest of $c'_=$ responses have the equivalent reward as $y'$, they can form a group of size $c'_=+1$, in which each of them has the same weight. As a result, the total contribution of this group to the weight of $\nabla_\theta\pi_\theta(y|x)$ is 
\begin{equation}
\begin{aligned}
        &-\frac{k}{n-1}\frac{c'_=+1}{\tbinom{n-2}{k-2}}\frac{k-1}{c'_=+1}[\tbinom{c'_<+c'_=+1}{k-1}-\tbinom{c'_<}{k-1}]R(x,y')    \\
        =&-\frac{1}{\tbinom{n-2}{k-2}}\frac{k(k-1)}{n-1}[\tbinom{c'_<+c'_=+1}{k-1}-\tbinom{c'_<}{k-1}]R(x,y')   \\
        =&-\frac{1}{\tbinom{n-2}{k-2}}\frac{k(k-1)}{n-1}[\tbinom{c'_<}{k-2}+\tbinom{c'_<+1}{k-2}+...+\tbinom{c'_<+c'_=}{k-2}]R(x,y')   
\end{aligned}
\end{equation}
the estimator for the third term ($y'<y$) is the sum of every group. The last step ensures when this group of responses are sorted into non-decreasing order from index $c'_<+1$ to index $c'_<+c'_=+1$, the formula can be expressed using their index $-1$.

Next, we compute the estimator for the first ($y$) and second term ($y'=y$). Using Lemma~\ref{lemma:2}, we can also compute the contribution of first term to the weight of $\nabla_\theta\pi_\theta(y|x)$:
\begin{equation}
    \frac{1}{\tbinom{n-1}{k-1}}\frac{k}{c_=+1}[\tbinom{c_<+c_=+1}{k}-\tbinom{c_<}{k}]R(x,y)
\end{equation}
for $\sum_{t=1}^k P_{<,\theta}(y)^{t-1} P_{\leq,\theta}(y)^{k-t}$ in the first term. The $c_<$ and $c_=$ are the number of responses with rewards less than or equal to $y$ among the $n-1$ responses, respectively.

For the second term, we show that
\begin{lemma} \label{lemma:3}
    \begin{equation}
    \begin{aligned}
        &\sum_{0\leq a\leq m,b=m-a} (a+1)\frac{\tbinom{c_<}{a}\tbinom{c_<+c_=-a}{b}}{\tbinom{a+b}{a}}\\
        =&\frac{(m+1)(m+2)}{(c_=+1)(c_=+2)}\tbinom{c_<+c_=+2}{m+2}-\frac{(c_<+1)(c_<-m)}{c_=+1}\tbinom{c_<}{m}\\
        &+\frac{(c_<-m)(c_<-m-1)}{c_=+2}\tbinom{c_<}{m}
    \end{aligned}
    \end{equation}
\end{lemma}
Using this lemma, the contribution of second term to the weight of $\nabla_\theta\pi_\theta(y|x)$ is 
    \begin{equation}
    \begin{aligned}
    -\frac{c_=}{(n-1)\tbinom{n-2}{k-2}}[
        &\frac{(m+1)(m+2)}{(c_=+1)(c_=+2)}\tbinom{c_<+c_=+2}{m+2}-\frac{(c_<+1)(c_<-m)}{c_=+1}\tbinom{c_<}{m}\\
        &+\frac{(c_<-m)(c_<-m-1)}{c_=+2}\tbinom{c_<}{m}]R(x,y)
    \end{aligned}
    \end{equation}

As the first term and the second term all involves $R(x,y)$, we sum up their contribution together and obtain:
    \begin{equation}
    \begin{aligned}
    \frac{1}{\tbinom{n-1}{k-1}}\frac{k}{c_=+1}&[\tbinom{c_<+c_=+1}{k}-\tbinom{c_<}{k}]R(x,y)\\
    -\frac{c_=}{(n-1)\tbinom{n-2}{k-2}}&[
        \frac{(m+1)(m+2)}{(c_=+1)(c_=+2)}\tbinom{c_<+c_=+2}{m+2}-\frac{(c_<+1)(c_<-m)}{c_=+1}\tbinom{c_<}{m}\\
        &+\frac{(c_<-m)(c_<-m-1)}{c_=+2}\tbinom{c_<}{m}]R(x,y)\\
        =k\frac{\tbinom{c_<}{k-1}}{\tbinom{n-1}{k-1}}R(x,y)&&
    \end{aligned}
    \end{equation}    
Surprisingly, though the contribution of $R(x,y)$ is complex at the beginning, its final expression is very straightforward. 

Putting together all three terms, we have
\begin{theorem} \label{theorem_maxkreal}
When $n\ge k$, an unbiased estimator for Max@k is, 
\begin{equation}
\begin{aligned}
J_\text{max@k}^\text{RSPO}(\theta)=&    \frac{k}{n|\mathcal{D}_b|}\sum_{x\sim \mathcal{D}_b} \sum_{i=1}^n 
    [R(x,y^i)\frac{\tbinom{c_<[i]}{k-1}}{\tbinom{n-1}{k-1}}-\\
&\frac{k-1}{n-1}\sum_{j=1}^{c_<[i]}R(x,y^j)\frac{\tbinom{j-1}{k-2}}{\tbinom{n-2}{k-2}}]
    \nabla_\theta\log\pi_\theta(y^i|x)   
\end{aligned}
\end{equation}
where we sort $\{y^i\}_{i=1}^n$ by $R(x,y_1)\leq R(x,y_2)\leq ...\leq R(x,y_n)$ and set $c_<[i]=|\{j| R(x,y^j)<R(x,y^i)\}\bigr|$.
\end{theorem}

After combinatorial operations, Theorem~\ref{theorem_maxkreal} proposes a simple unbiased estimator for the Max@k object, which greatly simplifies the complex computation in Algorithm~\ref{alg:unbiased-grad}. 

Moreover, compared with the Max@k estimator in Section~\ref{sec:maxk} with approximations, the only difference is replacing $i-1$ by $c_<[i]$. This again verify that the estimator in Section~\ref{sec:maxk} is almost correct in the continues reward scenarios.

In addition, when rewards are binary, only $y^i$ with $R(x,y^i)=1$ has non-zero gradient weight, and the $c_<[i]$ is the number of responses with $R(x,y)=0$. This indicates that this Max@k estimator is exactly the Pass@k estimator of Theorem~\ref{theorem:passk}.

\newpage

\end{document}